\documentclass[sigconf, nonacm]{acmart}

\usepackage{amsmath,amsfonts}
\usepackage{graphicx}
\usepackage{textcomp}
\usepackage{xcolor}
\usepackage{booktabs}       % professional-quality
\usepackage{nicefrac}       % compact symbols for 1/2, etc.
\usepackage{microtype}      % microtypography
\usepackage{subfig}
\usepackage{xspace}
\usepackage{listings}
\usepackage{amsthm}
\usepackage{algorithm}
\usepackage[noend]{algpseudocode}
\usepackage{varwidth}
\usepackage{balance}
\usepackage{multirow}
\usepackage{bm}
\usepackage{dcolumn}
\usepackage{url}
\usepackage{cleveref}
\usepackage[tableposition=top]{caption}

\newcommand\vldbdoi{XX.XX/XXX.XX}
\newcommand\vldbpages{XXX-XXX}
\newcommand\vldbvolume{14}
\newcommand\vldbissue{1}
\newcommand\vldbyear{2020}
\newcommand\vldbauthors{\authors}
\newcommand\vldbtitle{\shorttitle} 
\newcommand\vldbavailabilityurl{https://github.com/FumiyukiKato/FL-TEE}
\newcommand\vldbpagestyle{plain}

\newcommand{\method}{\textsc{Olive}\xspace}
\newcommand{\red}[1]{\textcolor{red}{#1}}
\newcommand{\blue}[1]{#1}

\begin{document}
\title{\textsc{Olive}: Oblivious Federated Learning on Trusted Execution Environment Against the Risk of Sparsification}

%%
%% The "author" command and its associated commands are used to define the authors and their affiliations.
\author{Fumiyuki Kato}
\affiliation{%
  \institution{Kyoto University}
  \country{}
}
\email{fumiyuki@db.soc.i.kyoto-u.ac.jp}

\author{Yang Cao}
\affiliation{%
  \institution{Hokkaido University}
  \country{}
}
\email{yang@ist.hokudai.ac.jp}

\author{Masatoshi Yoshikawa}
\affiliation{%
  \institution{Kyoto University}
  \country{}
}
\email{yoshikawa@i.kyoto-u.ac.jp}

\begin{abstract}
Combining Federated Learning (FL) with a Trusted Execution Environment (TEE) is a promising approach for realizing privacy-preserving FL, which has garnered significant academic attention in recent years.
Implementing the TEE on the server side enables each round of FL to proceed without exposing the client's gradient information to untrusted servers.
This addresses usability gaps in existing secure aggregation schemes as well as utility gaps in differentially private FL.
However, to address the issue using a TEE, the vulnerabilities of server-side TEEs need to be considered---this has not been sufficiently investigated in the context of FL.
The main technical contribution of this study is the analysis of the vulnerabilities of TEE in FL and the defense.
First, we theoretically analyze the leakage of memory access patterns, revealing the risk of sparsified gradients, which are commonly used in FL to enhance communication efficiency and model accuracy.
Second, we devise an inference attack to link memory access patterns to sensitive information in the training dataset.
Finally, we propose an oblivious yet efficient aggregation algorithm to prevent memory access pattern leakage.
Our experiments on real-world data demonstrate that the proposed method functions efficiently in practical scales.
\end{abstract}

\maketitle

\pagestyle{\vldbpagestyle}
\begingroup\small\noindent\raggedright\textbf{PVLDB Reference Format:}\\
\vldbauthors. \vldbtitle. PVLDB, \vldbvolume(\vldbissue): \vldbpages, \vldbyear.\\
\href{https://doi.org/\vldbdoi}{doi:\vldbdoi}
\endgroup
\begingroup
\renewcommand\thefootnote{}\footnote{\noindent
This work is licensed under the Creative Commons BY-NC-ND 4.0 International License. Visit \url{https://creativecommons.org/licenses/by-nc-nd/4.0/} to view a copy of this license. For any use beyond those covered by this license, obtain permission by emailing \href{mailto:info@vldb.org}{info@vldb.org}. Copyright is held by the owner/author(s). Publication rights licensed to the VLDB Endowment. \\
\raggedright Proceedings of the VLDB Endowment, Vol. \vldbvolume, No. \vldbissue\ %
ISSN 2150-8097. \\
\href{https://doi.org/\vldbdoi}{doi:\vldbdoi} \\
}\addtocounter{footnote}{-1}\endgroup

\ifdefempty{\vldbavailabilityurl}{}{
\vspace{.3cm}
\begingroup\small\noindent\raggedright\textbf{PVLDB Artifact Availability:}\\
The source code, data, and/or other artifacts have been made available at \url{\vldbavailabilityurl}.
\endgroup
}

\section{Introduction}
\label{sec:introduction}
In the current Big Data era, the challenge of preserving privacy in machine learning (ML) techniques has become increasingly apparent, as symbolized by the proposal of the GDPR \cite{gdpr}.
Federated learning (FL) \cite{mcmahan2016federated} is an innovative paradigm of privacy-preserving ML, which has been tested in production \cite{ramaswamy2019federated, paulik2021federated, google-FL-deploy}.
Typically, in FL, the server does not need to collect raw data from \textit{users} (we use \textit{participants} and \textit{clients} interchangeably)---it only collects \textit{gradients} (or model \textit{parameters} delta) trained on the local data of users during each round of model training. 
The server then aggregates the collected gradients into a global model.
Thus, FL is expected to enable data analyzers avoid the expenses and privacy risks of collecting and managing training data containing sensitive information.

\begin{figure}[t]
    \centering
    \includegraphics[width=1.02\hsize]{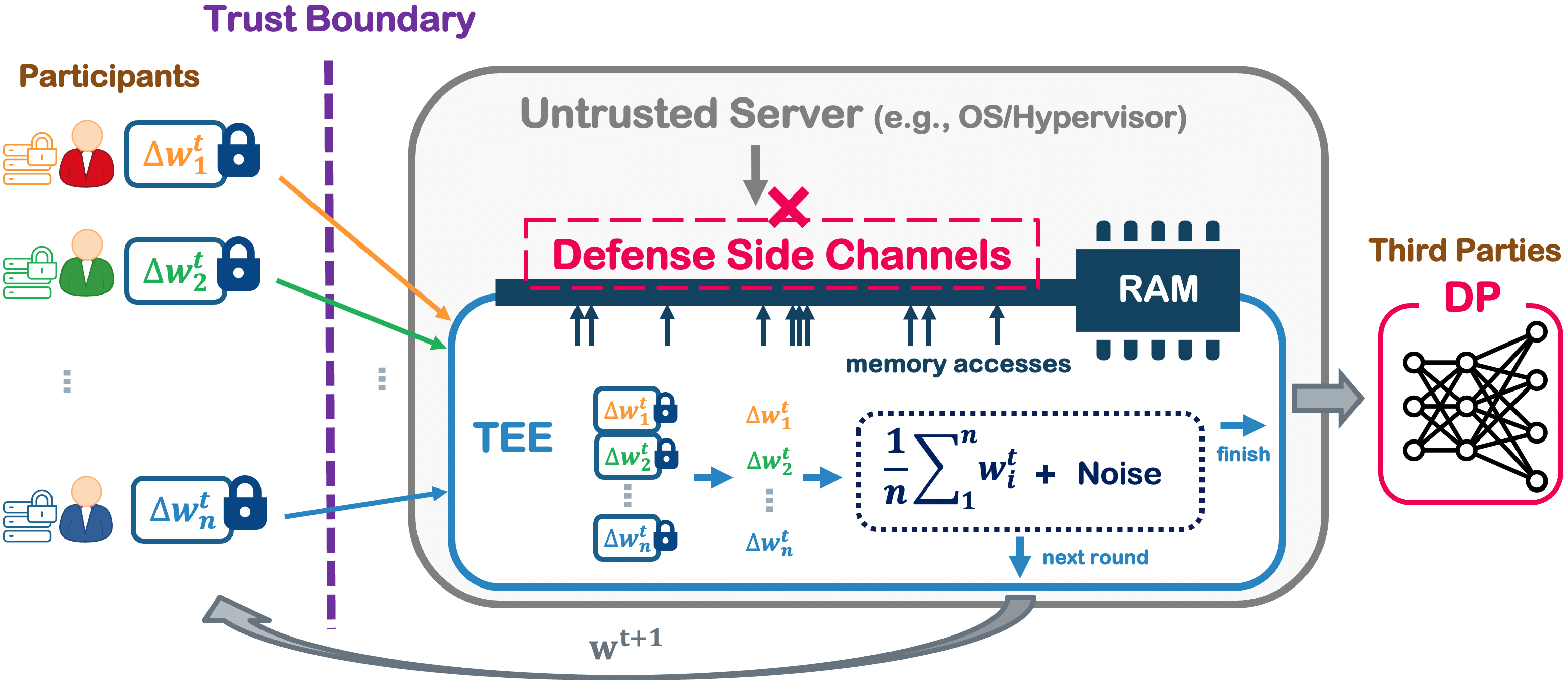}
    \caption{
    \method, i.e., {\bf O}b{\bf L}{\bf I}{\bf V}ious f{\bf E}derated learning on TEE is the first method of its kind to prevent privacy risks caused by the leakage of memory access patterns during aggregation in FL rigorously. This allows, for example, to enjoy utility of CDP-FL without requiring a trusted server like LDP-FL.}
    \label{fig:overview}
    \vspace{-3pt}
\end{figure}

However, multiple studies have highlighted the vulnerability of FL to various types of attacks owing to its decentralized scheme.
One of its most extensively studied vulnerabilities is an inference attack on a client's sensitive training data during the aggregation phase by an untrusted server \cite{zhu2020deep, zhao2020idlg, wainakh2022user, shokri2017membership, fu2022label}.
This attack arises from the requirement for each client to share raw gradient information with the central aggregation server in plain FL.
This creates the risk of privacy leakage from the training data, making it a vulnerable attack surface.
These attacks highlight the privacy/security problems of running FL on an untrusted server.

Enhancing FL using a Trusted Execution Environment (TEE) is a promising approach to achieve privacy-preserving FL, which has garnered significant attention in recent years \cite{zhao2021sear, zhang2021citadel, mo2021ppfl, zhang2021shufflefl, nguyen2022federated}.
TEE \cite{sabt2015trusted, costan2016intel} is a secure hardware technique that enables secure computation in an untrusted environment without exposing data or processing to the host (i.e., OS or hypervisor).
TEE guarantees confidentiality, integrity, verifiability, and functionalities such as remote attestation, fully justifying its use on the untrusted server side in FL \cite{zhang2021citadel, hunt2018chiron, zhang2021shufflefl}.
Gradients are transmitted to the TEE via a secure channel and computed securely in confidential memory, thereby eliminating the aforementioned attack surface.

Utilization of TEE is advantageous from several perspectives.
Although similar functionality is provided by secure aggregation (SA)\footnote{The recent paper \cite{mohamad2023sok} categorized TEE as a method of secure aggregation in FL.} based on pairwise masking, it sacrifices usability \cite{10.1145/3133956.3133982, ergun2021sparsified, DBLP:conf/icml/KairouzL021, lu2023top}. 
This requires time-consuming synchronous distributed mask generation among multiple clients and lacks robustness with respect to participant asynchronicity/dropouts \cite{mohamad2023sok}, which is difficult to handle and can impede implementation by general practitioners.
Further, SA is inflexible and makes it hard to do extensions, such as Byzantine resistance \cite{zhao2021sear} and asynchrony \cite{nguyen2022federated}.
In addition, application of gradient sparsification to FL with SA requires either random sparsification \cite{ergun2021sparsified} or a common sparsified index among multiple clients \cite{lu2023top} because of the pairwise constraints, impairing training quality.
One simple and important solution to these problems is the use of a TEE, even though it requires additional special hardware.

In addition, FL with TEE addresses the utility gap of differentially private FL (DP-FL) \cite{geyer2017differentially, DBLP:conf/iclr/McMahanRT018, erlingsson2020encode}.
The recently studied Shuffle DP-FL \cite{liu2020flame, girgis2021shuffled, erlingsson2020encode}, which aims to combine the best LDP-FL trust model  \cite{zhao2020local, ldp2020} with the model utility of the CDP-FL \cite{geyer2017differentially, DBLP:conf/iclr/McMahanRT018, andrew2021nips}, exhibits a gap with respect to CDP-FL in terms of utility \cite{erlingsson2020encode}.
As depicted in Figure \ref{fig:overview}, TEE facilitates secure model aggregation on an untrusted server, which ensures only differentially private models are observable by the server.
Without trust in the server, as in LDP-FL, model utility is equivalent to that of conventional CDP-FL because any DP mechanism can be implemented within the TEE, whereas the mechanism is restricted when using SA \cite{DBLP:conf/icml/KairouzL021}.
This important use case, i.e., the combination of the proposed method with CDP-FL, is analyzed in detail in Appendix D.

However, implementing a server-side TEE to achieve the aforementioned benefits requires careful analysis of the vulnerabilities of TEE.
Several serious vulnerabilities are known to affect TEE owing to side-channel attacks \cite{xu2015controlled, van2017telling, nilsson2020survey}, which can cause privacy leakage despite encryption.
In particular, such attacks can expose data-dependent memory access patterns of confidential execution and enable attackers to steal sensitive information, such as RSA private keys and genome information \cite{brasser2017software}.
The specific information that may be stolen from these memory access patterns is domain-specific and is not yet known for FL, although several studies have attempted to use TEE for FL \cite{mo2021ppfl, zhang2021shufflefl, 10.1145/3477114.3488765, zhang2021citadel, flatee2021}.
Thus, the extent of the threat of side-channel attacks against FL on a TEE and the types of possible attacks remain critical open problems in this context.

\textit{Oblivious algorithms} \cite{goldreich1987towards, pathoram2013ccs, ohrimenko2016oblivious} are important leakage prevention techniques that generate only data-independent memory access patterns. 
A general approach involves making the RAM oblivious, e.g., oblivious RAM (ORAM). 
PathORAM \cite{pathoram2013ccs} is known to be the most efficient technique.
However, it assumes a private memory space of a certain size and is not applicable to practical TEE, such as Intel SGX \cite{costan2016intel}.
Although Zerotrace \cite{DBLP:conf/ndss/SasyGF18} addresses this issue, its still incurs significant overhead.
Therefore, the design of an algorithm-specific method to obtain an efficient algorithm is an important problem.
In this context, \cite{ohrimenko2016oblivious} proposed an efficient oblivious algorithm for specific ML algorithms, and \cite{zheng2017opaque} studied SQL processing.
However, an efficient method for FL-specific aggregation algorithm, which can be a vulnerable component of FL with a server-side TEE, has not yet been proposed.

In this study, we address the aforementioned gaps; (1) we clarify privacy risks by designing specific attacks on FL with a server-side TEE and demonstrate them in a real-world scenario; (2) we devise a novel defense against the risks by designing efficient oblivious algorithms and evaluate them empirically on a practical scale.
Our analysis reveals that parameter position information is leaked during the execution of the FL aggregation algorithm in a \textit{sparsified} environment.
\textit{Sparsification} is often used in FL \cite{lin2018dgc, sahu2021rethinking, ergun2021sparsified, lu2023top} to reduce communication costs and/or improve model accuracy \cite{aledhari2020federated}.
The goal of an attacker is to infer a set of sensitive labels included in the target user's training data, similar to the goal described in \cite{wainakh2022user, fu2022label}.
We assume the attacker to be capable of observing memory access patterns, accessing the dataset that covers the overall dataset distribution, and accessing the model trained during each round.
Although sparsified index information in FL has been considered as somewhat private information in previous studies \cite{lu2023top, liu2020flame}, unlike in our study, no specific attacks have been investigated.
After demonstrating the proposed attack on real-world datasets, we propose efficient oblivious algorithms to prevent such attacks completely.
To this end, we carefully construct existing oblivious building blocks, such as the oblivious sort \cite{batcher1968sorting} and our designed components.
Our proposed method \method, an {\bf O}b{\bf L}{\bf I}{\bf V}ious f{\bf E}derated learning system based on server-side TEE, is resistant to side-channel attacks, enabling truly privacy-preserving FL.
In addition to fully oblivious algorithms, we further investigate optimization by adjusting the data size in the enclave, and study more efficient algorithms by relaxing the definition of obliviousness.
Finally, we conduct extensive experiments on real-world data to demonstrate that the proposed algorithm, designed for FL aggregation, is more efficient than the general-purpose PathORAM with SGX \cite{DBLP:conf/ndss/SasyGF18}.

The contributions of this study are summarized below:
\begin{itemize}
    \item We analyze the exposure of memory access patterns to untrusted servers when TEE is used for model aggregation in FL. A risk is identified in the context of sparsified gradients, which are often used in recent FL.
    \item We design a supervised learning-based sensitive label inference attack based on index information observed from side-channels of sparsified gradients. We demonstrate the attack on a real-world dataset. One of the results reveals that when training with a CNN on CIFAR100 with top-1.25\% sparsification, the sensitive labels of training data (each participant is assigned 2 out of 100 labels) are leaked with approximately 90\% or better accuracy (Figure \ref{fig:attack_sparse_ratio}). 
    \item We propose a novel oblivious algorithm that executes model aggregation efficiently by combining oblivious primitives, such as oblivious sort and certain designed components. The efficiency of the proposed method is verified via extensive experiments. In particular, it is demonstrated to be more than 10 $\times$ faster than a PathORAM-based method and require only a few seconds even in cases involving a million parameters (Figure \ref{fig:performance_on_artificial_data}).
\end{itemize}

The remainder of this paper is organized as follows. 
Preliminary notions are presented in Section \ref{sec:background}.
The overview of the proposed system and the problem setting is described in Section \ref{sec:problem_definition}.
Sections \ref{sec:attack_on_index} and \ref{sec:defense} demonstrate the proposed attack and defense, respectively, with empirical evaluations.
Section \ref{sec:related_works} discusses related works and Section \ref{sec:conclusion} concludes.
The details of the combination of DP and the proposed \method are provided in Appendix D.

\section{Preliminaries}
\label{sec:background}
\subsection{Federated Learning}
\label{sec:background_federated_learning}
Federated learning (FL) \cite{mcmahan2016federated, konevcny2016federated} is a recent ML scheme with distributed optimization.
The basic FL algorithm, called \texttt{FedAVG} \cite{mcmahan2016federated}, trains models by repeating model optimization steps in the local environment of the participants and updating the global model by aggregating the parameters of the locally trained models.
\texttt{FedSGD} \cite{mcmahan2016federated} exchanges locally updated gradients based on distributed stochastic gradient descent.
Overall, users are not required to share their training data with the server, which represents a major advantage over traditional centralized ML.

\textbf{Sparsification.}
To reduce communication costs and improve model accuracy, the sparsification of the model parameters before their transmission to the server has been extensively studied in FL \cite{lin2018dgc, sahu2021rethinking, ergun2021sparsified, lu2023top, hu2022federated, wu2022smartidx, Shahid2021CommunicationEI}.
All of the aforementioned methods sparsify parameters on the client side, apply an encoding that represents them as \textit{value} and \textit{index} information \cite{wu2022smartidx}, transmit them to the server, and aggregate them into a dense global model on the server side.
Exceptionally, \cite{hu2022federated, lu2023top} used common sparsification among all clients using common sparsified indices and aggregated them into a sparse global model.
However, as observed in \cite{ergun2021sparsified}, there is practically little overlap among the top-$k$ indices for each client in real-world data, especially in the non-i.i.d. environment, which is common in FL.
This highlights the one of limitations of pairwise masking-based SA \cite{ergun2021sparsified, lu2023top} (see Section \ref{sec:related_works}).
In general, top-$k$ sparsification is the standard method.
By transmitting only the top-$k$ parameters with large absolute gradients to the aggregation server, communication cost is reduced by more than 1 \textasciitilde 3 orders of magnitude \cite{sahu2021rethinking}.
This technique outperforms the random selection of $k$ indices (random-$k$) \cite{ergun2021sparsified}, particularly when the compression ratio is smaller than 1\% \cite{sahu2021rethinking, wu2022smartidx, hu2022federated, lu2023top}.
Other sparsification methods, such as threshold-based \cite{sahu2021rethinking}, top-$k$ under LDP \cite{liu2020fedsel} and the recently proposed convolutional kernel \cite{wu2022smartidx}, also exist.
However, these sparsified gradients can lead to privacy leakages through the index.
In \cite{lu2023top, liu2020flame}, the set of user-specific top-$k$ indices was treated as private information; however, no specific attacks were investigated.

\subsection{Trusted Execution Environment}
\label{sec:tee}
The TEE, as defined formally in \cite{sabt2015trusted}, creates an isolated execution environment within untrusted computers (e.g., cloud VMs).
We focus on a well-known TEE implementation---Intel SGX \cite{costan2016intel}.
It is an extended instruction set for Intel x86 processors, which enables the creation of an isolated memory region called an \textit{enclave}.
The enclave resides in an encrypted and protected memory region called an \textit{EPC}.
The data and programs in the EPC are transparently encrypted outside the CPU package by the Memory Encryption Engine, enabling performance comparable to native performance.
SGX assumes the CPU package to be the trust boundary---everything beyond it is considered untrusted---and prohibits access to the enclave by any untrusted software, including the OS/hypervisor.
Note that for design reasons, the user-available size of the EPC is limited to approximately 96 MB for most current machines.
When memory is allocated beyond this limit, SGX with Linux provides a special paging mechanism.
This incurs significant overhead for encryption and integrity checks, resulting in poor performance \cite{maliszewski2021price, taassori2018vault, kato2021pct}.

\textbf{Attestation.} SGX supports remote attestation (RA), which can verify the correct initial state and genuineness of an enclave.
On requesting the RA, a report with measurements based on the hash of the initial enclave state generated by the trusted processor is received.
This facilitates the identification of the program and completes the memory layout.
Intel EPID signs this measurement and the Intel Attestation Service verifies the correctness of the signature as a trusted third party.
Consequently, verifiable and secure computations are performed in a remote enclave.
Simultaneously, a secure key exchange is performed between the enclave and the remote client within this RA protocol.
Therefore, after performing RA, communication with a remote enclave can be initiated over a secure channel using AES-GCM and so on.

\subsection{Memory Access Pattern Leakage}
\label{sec:memory_access_pattern_leak}
Although the data are encrypted and cannot be viewed in enclaves, memory/page access patterns or instruction traces can be exposed irrespective of the use of a TEE \cite{xu2015controlled, van2017telling, brasser2017software, lee2017inferring, nilsson2020survey}.
This may lead to sensitive information being stolen from enclaves \cite{brasser2017software}.
For example, cacheline-level access pattern leakage occurs when a malicious OS injects page faults \cite{xu2015controlled} or uses page-table-based threats \cite{van2017telling, nilsson2020survey}.
Moreover, if a physical machine is accessible, probes may be attached to the memory bus directly.

To prevent such attacks, oblivious algorithms have been proposed to hide access patterns during the secure execution of the process.
An oblivious algorithm is defined as follows.
\begin{definition}[Oblivious algorithm \cite{chan2019foundations}]
An algorithm $\mathcal{M}$ is $\delta$-statistically oblivious if, for any two input data $I$ and $I'$ of equal length and any security parameter $\lambda$, the following relation holds:
\begin{equation}
\nonumber
  \mathbf{Accesses}^{\mathcal{M}}(\lambda, I) \overset{\delta(\lambda)}{\equiv} \mathbf{Accesses}^{\mathcal{M}}(\lambda, I')
\end{equation}
where $\mathbf{Accesses}^{\mathcal{M}}(\lambda, I)$ denotes a random variable representing the ordered sequence of memory accesses.
The algorithm $\mathcal{M}$ is generated upon receiving the inputs, $\lambda$ and $I$.
$\overset{\delta(\lambda)}{\equiv}$ indicates that the statistical distance between the two distributions is at most $\delta(\lambda)$.
The term $\delta$ is a function of $\lambda$ which corresponds to a cryptographic security parameter.
When $\delta$ is negligible, we say that $\mathcal{M}$ is \textit{fully oblivious}, and when $\delta$ is 1, it is \textit{not oblivious}.
\end{definition}

A typical approach for constructing an oblivious algorithm utilizes an ORAM, such as PathORAM \cite{pathoram2013ccs}.
Although ORAMs are designed for general use as key-value stores, several oblivious task-specific algorithms, such as ML \cite{ohrimenko2016oblivious} and SQL processing \cite{zheng2017opaque} (see Section \ref{sec:related_works} for details), have been proposed from a performance perspective.
They are constructed based on oblivious sort \cite{batcher1968sorting} and/or access to all memory (i.e., linear scan), and are distinct from ORAM at the algorithmic level.
Further, ORAM generally assumes that the existence of a trusted memory space such as \textit{client storage} \cite{pathoram2013ccs}, which is incompatible with the SGX assumption of leaking access patterns in enclaves.
Thus, only CPU registers should be considered to be trusted memory spaces \cite{DBLP:conf/ndss/SasyGF18}.
\cite{ohrimenko2016oblivious} implemented oblivious ML algorithms using \texttt{CMOV}, which is an x86 instruction providing a conditional copy in the CPU registers.
\texttt{CMOV} moves data from register to register based on a condition flag in the register, which is not observed by any memory access patterns.
Using the \texttt{CMOV} instruction, conditional branching can be implemented with a constant memory access pattern that does not depend on the input, thereby removing the leakage of subsequent code addresses.
For example, Zerotrace \cite{DBLP:conf/ndss/SasyGF18} implements PathORAM on SGX by obliviously implementing client storage based on \texttt{CMOV}.
We can construct and use low-level oblivious primitives, such as \textit{oblivious move} (\texttt{o\_mov}, Listing 1) and \textit{oblivious swap} (\texttt{o\_swap}, Listing 2).
\texttt{o\_mov(flag,x,y)} is a function that accepts a Boolean condition flag as its first argument and returns \texttt{x} or \texttt{y} depending on the flag.
Therefore, designing an appropriate oblivious algorithm for SGX requires a combination of high-level algorithm designs, such as the oblivious sort and low-level primitives.

\section{Proposed System}
\label{sec:problem_definition}
In this section, we first clarify our scenario and threat model, and then present a system overview of the \method.
Finally, we analyze the details of the potential privacy risk, followed by discussion of a specific privacy attack and evaluation in Section \ref{sec:attack_on_index}.

\subsection{Scenario}
\label{sec:problem_scenario}
We target a typical FL scenario with a single server and clients using identical format data (i.e., horizontal FL).
The server is responsible for training orchestration, aggregating parameters, updating the global model, selecting clients for each training round, and validating model quality.
The server-side machine is assumed to be placed in a public or private environment \cite{zhang2021citadel, hunt2018chiron} and is equipped with a TEE capable of RA (e.g., Intel SGX).

\textbf{Threat model.} We assume an adversary to be a \textit{semi-honest} server that allows FL algorithms to run as intended, while trying to infer the sensitive information of clients based on shared parameters.
This is a compatible threat model with those in existing studies on FL with SA \cite{10.1145/3133956.3133982} and even with server-side TEE \cite{zhang2021citadel, mo2021ppfl, zhang2021shufflefl}.
The semi-honest threat model is selected despite using TEE, because the assumed attack in this work does not diverge from the established FL protocol.
The goal of the adversary is not to damage the availability (e.g., DoS attacks) or undermine the utility of the model (e.g., data-poisoning attacks) \cite{10.1145/3322205.3311083, zhao2021sear, pmlr-v108-bagdasaryan20a} as \textit{malicious} attackers in FL context.
Note that several side-channel attacks against TEE require malicious (i.e., privileged) system software, which we distinguish from an attacker and categorize as \textit{malicious} in FL.
Nevertheless, \cite{boenisch2021curious} reported that malicious servers improve inference attacks in FL.
In Section \ref{sec:defense_discussion}, we discuss the relationship between such malicious servers and the privacy and security of the proposed system.

We assume that the server has (1) access to the trained model during each round of FL, \blue{(2) access to the global test dataset}, and (3) the capability to observe the memory access patterns of the TEE.
These requirements can be justified as follows.
(1): Because the server is in charge of model validation, it makes sense for the server to have access to the global models during all rounds.
Alternatively, attackers can easily blend in with clients to access global models.
(2): Generally, the semi-honest server that has access to public datasets for model validation covers the overall dataset distribution, which is essential in production uses.
Similar assumptions have been made in previous studies on inference attacks \cite{wang2019beyond, hu2022federated}.
Subsequently, we experimentally evaluate the required dataset volume (Figure \ref{fig:attack_reducing_test_dataset}).
(3): This follows the general threat assumption for TEE.
The SGX excludes side-channel attacks from the scope of protection \cite{costan2016intel, nilsson2020survey}.
Except for the trusted hardware component (i.e., the CPU package), all other components of the server, e.g., the system software (i.e., OS/hypervisor), main memory, and all communication paths, are considered to be untrusted.
The server can observe memory access patterns through known or unknown side-channel attacks, as described in Section \ref{sec:memory_access_pattern_leak}.

\subsection{System overview}
\label{sec:problem_overview}
The proposed system, namely the \method (Figure \ref{fig:overview}), follows basic \texttt{FedAVG} algorihtm with standard top-$k$ sparsification; however, the TEE is placed on the server side with a server-side algorithm resistant to side-channel attacks.
As an initial configuration, we provide an enclave in which each client verifies the integrity of the processes running on the enclave via RA and exchanges shared keys (AES-GCM).
If attestation fails, the client must refuse to join the FL in this phase.
We assume that communication between the client and server is performed over a secure channel (TLS), which the untrusted server terminates, and that the transmitted gradients\footnote{In \texttt{FedAVG}, the data shared by users are not exactly gradients---rather, they are the delta of model weights.
However, in the context of compatibility with \texttt{FedSGD}, we jointly refer to model update data transmitted by users as \textit{gradients} or \textit{parameters}.} are doubly encrypted and can only be decrypted in the trusted enclave.

The overall algorithm of the \method is presented in Algorithm \ref{alg:olive}, where the differences with respect to the basic \texttt{FedAVG} algorithm are highlighted in red.
The initial provisioning is omitted and a different shared key, $sk_i$, is stored in the enclave for each user, $i$ $(\in [N])$ (line 1).
In each round, the participants are securely sampled in the enclave (line 4).
The selected users are memorized in the enclave and used for client verification (line 9) after the encrypted data are loaded into the enclave (line 8).
On the client side, locally trained parameters are top-$k$ sparsified (line 21), and then encoded and encrypted (line 22).
The encrypted data loaded into the enclave are decrypted and verified (line 11).
Verification (lines 9, 11) is not essential to our work; however, it prevents man-in-the-middle attacks and biased client selection.
As discussed in Section \ref{sec:security_analysis}, the aggregation operation (line 12) is required to be oblivious, and we present lower-level and detailed algorithms in Section \ref{sec:defense} to this end.
In accordance with the principle that the Trusted Computing Base (TCB) should be minimized, only the aggregation operation is performed in the enclave.
Finally, the aggregated parameters are loaded outward from the enclave (line 13).
Thus, the parameters transmitted by all clients remain completely invisible to the server,---only the aggregated parameters are observable.

\begin{algorithm}[t]
\small
\caption{\method: Oblivious FL on TEE}
\label{alg:olive}
\begin{algorithmic}[1]
\renewcommand{\algorithmicrequire}{\textbf{Input:}}
\renewcommand{\algorithmicensure}{\textbf{Output:}}

\Require $N$: \# participants, $\eta_c$, $\eta_s$: learning rate
\State \red{$\text{KeyStore} \leftarrow \textit{Remote Attestation}$ with all user $i$} \Comment{key-value store in enclave that stores $sk_{i}$: user $i$'s shared key from RA in provisioning}
\Procedure{Train}{$q$, $\eta_c$, $\eta_s$}
    \State Initialize model $\theta^0$
    \For{each round $t=0, 1, 2, \ldots$}
        \State \red{$\mathcal{Q}^t \leftarrow$ (sample users from $N$ for round $t$)} \Comment{securely in enclave}
        \For{each user $i \in \mathcal{Q}^t$ \textbf{in parallel}}
            \State $\text{Enc}(\Delta^{t}_i) \leftarrow$ \textsc{EncClient}$(i, \theta^t, \eta_c)$
            \State \red{$\text{LoadToEnclave}(\text{Enc}(\Delta^{t}_i))$}
            \State \red{check if user $i$ is in $\mathcal{Q}^t$}
            \State \red{$sk_{i} \leftarrow \text{KeyStore}[i]$} \Comment{retrieve user $i$'s shared key}
            \State \red{$\Delta^{t}_i \leftarrow \text{Decrypt}(Enc(\Delta^{t}_i), sk_{i})$}
        \EndFor
        \State \red{/* \textbf{Obliviously performed, such as Algorithm \ref{alg:baseline} or \ref{alg:advanced}} */}
        \Statex \;\;\;\;\;\;\;\;\;\, \red{$\mathbf{\tilde{\Delta}^t = \frac{1}{|\mathcal{Q}^t|}\sum_{i \in \mathcal{Q}^t} \Delta^{t}_i}$} \Comment{\textbf{oblivious algorithm}}\label{alg1:dp}
        \State \red{$\text{LoadFromEnclave}(\tilde{\Delta}^t)$}
        \State $\theta^{t+1} \leftarrow \theta^t + \eta_s \bar{\Delta}^t$
    \EndFor
\EndProcedure

\Procedure{EncClient}{$i$, $\theta^t$, $\eta$, $C$}
    \State $\theta \leftarrow \theta^t$
    \State $\mathcal{G} \leftarrow $ (user $i$'s local data split into batches)
    \For{batch $g \in \mathcal{G}$}
        \State $\theta \leftarrow \theta - \eta \nabla \ell(\theta;g)$
    \EndFor
    \State $\Delta \leftarrow \theta - \theta^t$
    \State \red{$\Delta$ $\leftarrow$ TopkSparse$(\Delta)$}  \Comment{top-$k$ sparsification on gradients} \label{alg1:sparsification}
    \State \red{$\text{Enc}(\Delta') \leftarrow \text{Encrypt}(\Delta, sk_{i})$} \Comment{Authenticated Encryption (AE) mode, such as AES-GCM, with shared key, $sk_i$, from RA}
    \State \textbf{return} $\text{Enc}(\Delta)$
\EndProcedure

\end{algorithmic}
\end{algorithm}

\subsection{Security Analysis}
\label{sec:security_analysis}
Although TEE enables model training while protecting raw gradients, an untrusted server can observe the memory access patterns, as described in Section \ref{sec:memory_access_pattern_leak}.
Here, we analyze the threats that exist based on memory access patterns.

For formal modeling, let $g_i$ denote the $k$-dimensional gradient transmitted by user $i$ and let $g^*$ be the $d$-dimensional global parameter after aggregation.
In the typical case, $k=d$, when dense gradients are used.
Let $\mathbf{G}_i$ and $\mathbf{G^*}$ denote the memories required to store the gradients of $g_i$ and $g^*$, respectively, and let the number of clients participating in each round be $n$.
The memory that stores the entire gradient is denoted by $\mathbf{G}=\mathbf{G}_1 \mathbin\Vert ... \mathbin\Vert  \mathbf{G}_n$, where $\mathbin\Vert$ denotes concatenation.
A memory access, $a$, is represented as a triple $a=(\mathbf{A}[i], \texttt{op}, \texttt{val})$, where $\mathbf{A}[i]$ denotes the $i$-th address of the memory, $\mathbf{A}$; $\texttt{op}$ denotes the operation for the memory---either \texttt{read} or \texttt{write}; and \texttt{val} denotes the value to be written when $\texttt{op}$ is \texttt{write}, and \texttt{null} otherwise.
Therefore, the observed memory access pattern, $\mathbf{Accesses}$, can be represented as $\mathbf{Accesses}=[a_1, a_2,..., a_m]$ when the length of the memory access sequence is $m$.

\begin{figure}[t]
    \centering
    \includegraphics[clip, width=3.0in]{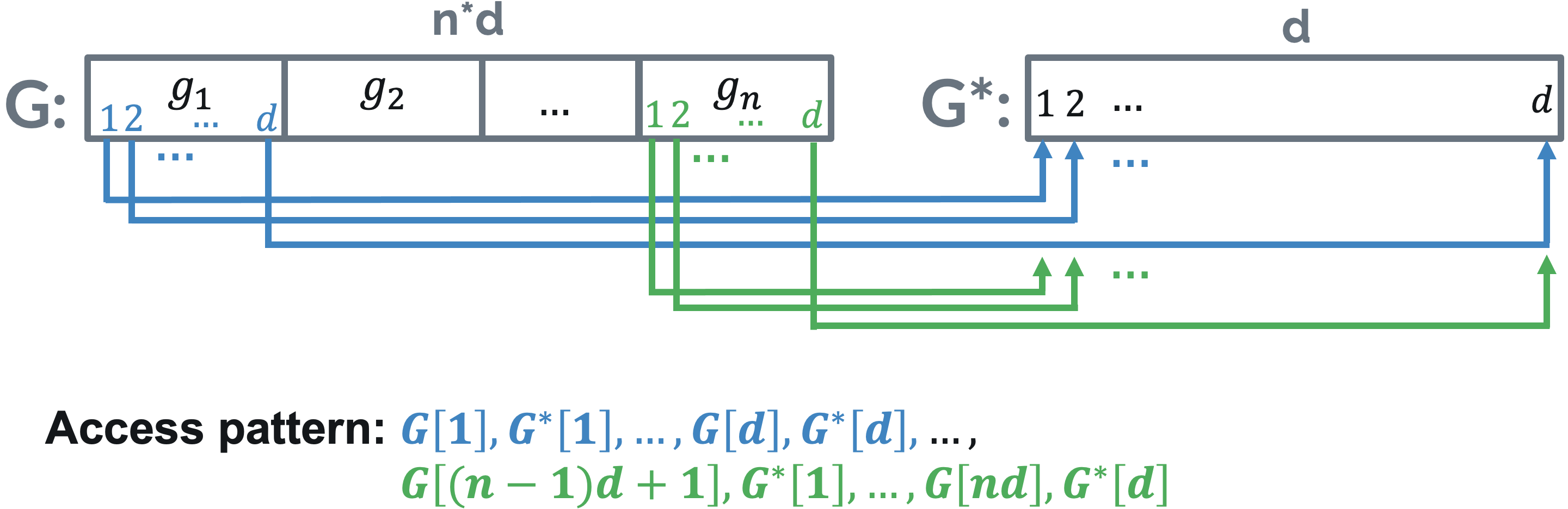}
    \vspace{-3px}
    \caption{Dense gradients induce uniform access patterns.}
    \label{fig:dense_g}
\end{figure}

\begin{figure}[t]
    \centering
    \includegraphics[clip, width=3.0in]{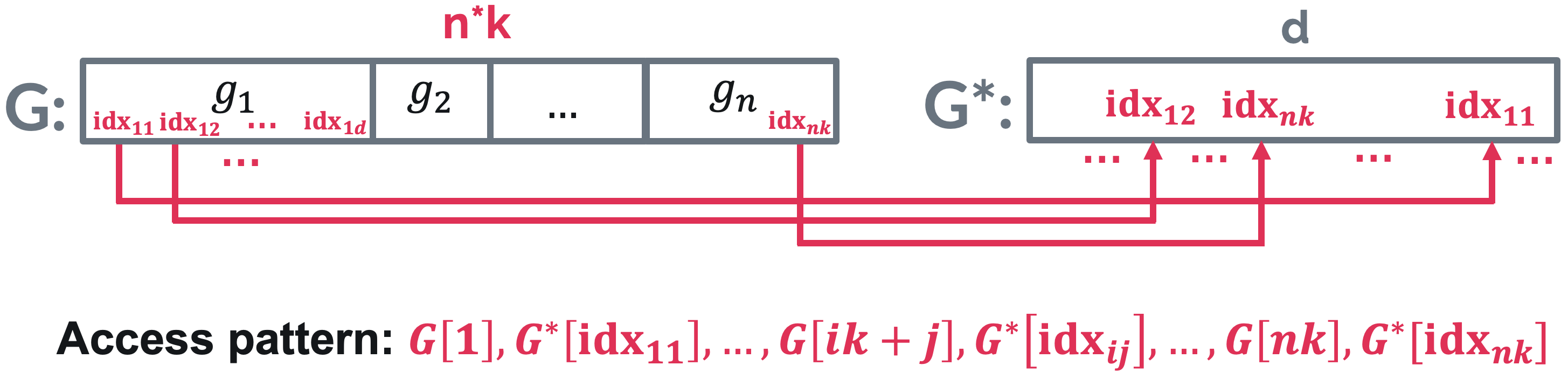}
    \vspace{-3px}
    \caption{Sparse gradients induce biased access patterns.}
    \label{fig:sparse_g}
\end{figure}

In FL, operations performed on the server side generally consist of summing and averaging the gradients obtained from all users.
We first note that this procedure is oblivious to \textit{dense gradients}.
As depicted in Figure \ref{fig:dense_g}, the summing operation involves updating the value of the corresponding index of $\mathbf{G^*}$ while performing a linear scan on $\mathbf{G}$, where memory accesses are performed in a fixed order and at fixed addresses, irrespective of the content of $\mathbf{G}$.
We refer to this general summing part as the \textit{linear} algorithm and present it in Appendix B for completeness.

\begin{proposition}
\label{prop:linear}
The \textit{linear} algorithm is fully oblivious to dense gradients. {\normalfont (An formal proof is presented in Appendix).}
\end{proposition}

\noindent
The linear algorithm is executed in $O(nd)$ because all the elements of the gradient $\mathbf{G}$ are accessed.
In addition, the averaging operation only accesses $\mathbf{G^*}$ linearly in $O(d)$, which is obviously fully oblivious.

However, when the gradients are \textit{sparsified}, which is often an important scenario in FL, the access pattern of the \textit{linear} algorithm is not oblivious, and sensitive information may be leaked.
The weights of sparse gradients are generally given by tuples of \texttt{index}, which hold the location information of the parameter, and a \texttt{value}, which holds the gradient value.
This is irrespective of its quantization and/or encoding because it requires calculating the sum of the original dense gradients.
Figure \ref{fig:sparse_g} depicts the access pattern when an aggregation operation is used for sparsified gradients.

\begin{proposition}
\label{th:2}
The \textit{linear} algorithm is \textbf{not} oblivious to sparsified gradients.
\end{proposition}

\begin{proof}
Linear access to \textbf{$\mathbf{G}$} for sparsified gradients occurs when the access pattern, $\mathbf{Accesses^{sparse}}$, satisfies 
\begin{equation}
\nonumber
\small
\begin{split}
  \mathbf{Accesses^{sparse}} &= \\
  [(\mathbf{G}[1], \mathtt{read}&,*), (\mathbf{G^*}[\mathrm{idx}_{11}], \mathtt{read},*), (\mathbf{G^*}[\mathrm{idx}_{11}], \mathtt{write},*),..., \\
  (\mathbf{G}[nk], \mathtt{read}&,*), (\mathbf{G^*}[\mathrm{idx}_{nk}], \mathtt{read},*), (\mathbf{G^*}[\mathrm{idx}_{nk}], \mathtt{write},*)]
 \end{split}
\end{equation}
\noindent
where the \texttt{index}es of sparsified gradients of user $i$ are $\mathrm{idx}_{i1},..,\mathrm{idx}_{ik}$ for all $i\in [n]$.
The access pattern, $\mathbf{Accesses^{sparse}}$, is deterministic and corresponds in a one-to-one fashion with the sequence of the \texttt{index}es of the input data.
Considering two input data, $I$ and $I'$, with different sequences of \texttt{index}es, no overlap exists in the output distribution.
Then, the statistical distance between them is $1$.
\end{proof}

\noindent
The access pattern on the aggregated gradients, \textbf{$\mathbf{G^*}$}, reveals at least one set of indices $\{\mathrm{idx}_{ij}\, |\, j \in [d]\}$ for each user $i$, depending on the given gradients.
Considering data-dependent sparsifications, such as top-$k$, which are generally used in FL, the gradient indices of the sparsified gradients may be sensitive to the training data.
In the next section, we demonstrate that privacy leakage can be caused on a real-world dataset.

\textbf{Generality and Limitation.} 
Let us now clarify the format and method of sparsified gradients.
Although various quantization and/or encoding methods in FL have been studied(e.g., \cite{sattler2019robust}), quantization is irrelevant to the problem of leakage considered in this study because it affects only the values and not the index, and encoding is irrelevant because it is eventually decoded on the server side.
For example, in \cite{ergun2021sparsified, lu2023top}, the index location information was encoded in $d$-dimensional one-bit array, but the same problem occurred during aggregation.
As aggregation is performed on the original dense gradients, each update requires access to a specific index of the dense gradients ($\mathbf{G^*}$), resulting in identical access patterns.
It should also be noted that risk is sparsification-dependent.
If the client's training data and observed indices are uncorrelated, then index leakage is not considered to be a risk.
For example, when random-$k$ is adopted, as in \cite{ergun2021sparsified}, no risk is involved.
While threshold-based sparsification \cite{sahu2021rethinking} is almost identical to top-$k$, LDP-guaranteed index \cite{liu2020fedsel} and the recently proposed convolution-kernel-based index \cite{wu2022smartidx} are still unclear.
These index information can correlate to some extent with the client's training data, but not as much as top-$k$.
The scope of our study is limited to the demonstration that attacks are possible with the standard top-$k$---the investigation of various other sparsifications are left for future research.

\section{Attack on gradient index}
\label{sec:attack_on_index}

\subsection{Design}
\label{sec:attack_design}

In this section, we design a server-side attack to demonstrate that privacy leakage of the training data can occur based on the index information in the gradients.
We assume a sparsified gradient based on top-$k$ \cite{shi2019convergence, lin2018dgc, sahu2021rethinking}.
The attacker is assumed to satisfy the assumptions listed in Section \ref{sec:problem_scenario}.
The proposed attacks can be used to raise awareness of the security/privacy risks of FL on TEE, which have not been reported in related works \cite{flatee2021, mo2021ppfl, 10.1145/3477114.3488765, zhang2021citadel}, and also serve as an evaluation framework for defenses.

\begin{algorithm}[t]
\small
\caption{Attack on index: \textsc{Jac} or \textsc{NN}}
\label{alg:attack_on_index}
\begin{algorithmic}[1]
\renewcommand{\algorithmicrequire}{\textbf{Input:}}
\renewcommand{\algorithmicensure}{\textbf{Output:}}

\Require $i$: target user, $X_l$: test data with label $l$ $(l\in L)$, round: $T$
\State $\mathbf{index} \leftarrow \{\}$ \Comment{observed access patterns}
% \State $\mathbf{teacher} \leftarrow \{\mathbf{teacher}[l] |\; l\in L \}$ \Comment{teacher access patterns}
\State /* Prepare teacher and target indices */
\State $\mathbf{teacher} \leftarrow \{\}$ \Comment{teacher access patterns to train a classifier}
\For{each round $t=1, \ldots, T$}
    \State /* $T_i$: rounds participated in by user $i$ */
    \If{$t \in T_i$}
        \State /* $A^{(t)}_i$: observed top-$k$ indices of user $i$ of round $t$ */
        \State Store $A^{(t)}_i$ to $\mathbf{index}[i, t]$
        \For{each label $l\in L$}
            \State /* $\theta^{t}$: the global model after round $t$ */
            \State /* $I^{(t)}_l$: top-$k$ indexes training with $\theta^{t}$ and $X_l$ */
            % through training algorithm as the same local learning algorithm as \textsc{EncClient} with model $\theta^{(t)}$ and data $X_{test}[l]$ */
            \State Store $I^{(t)}_l$ to $\mathbf{teacher}[l, t]$
        \EndFor
    \EndIf
\EndFor
\State /* Calculate scores for each label $l$ */
\State  $\mathbf{S} \leftarrow$ [] \Comment{form of [(label, similarity)]}
\State /* If \textsc{Jac}: Jaccard similarity-based scoring (\textsc{Sim}) */
\For{each label $l\in L$}
    \State Store ($l$, \textsc{Sim}$(\mathbin\Vert_{\tau \in T_i}{\mathbf{index}[i, \tau]}, \mathbin\Vert_{\tau \in T_i}{\mathbf{teacher}[l, \tau]})$ to $\mathbf{S}$
\EndFor
\State /* If \textsc{NN}: neural network-based scoring */
\State Train the model $M_t$ with $\mathbf{teacher}[l, t]$ $(l\in L)$ for each $t$ $(\in T)$
\For{each label $l\in L$}
    \State Store ($l$, \textsc{predict}$(M_1,...,M_T, \mathbin\Vert_{\tau \in T_i}{\mathbf{index}[i, \tau]}))$ to $\mathbf{S}$
\EndFor
\State /* If \textsc{NN-single}: using single neural network */
\State Train the model $M_0$ with $\mathbin\Vert_{\tau \in T}{\mathbf{teacher}[l, \tau]}$ $(l\in L)$ 
\For{each label $l\in L$}
    \State Store ($l$, \textsc{predict}$(M_0, \mathbin\Vert_{\tau \in T}{\mathbf{index}[i, \tau]}))$ to $\mathbf{S}$
\EndFor
\State /* 1D K-Means clustering \textsc{Kmeans} */
\State [labels, centroid] $\leftarrow$ \textsc{Kmeans}$(\mathbf{S})$
\State \textbf{return} labels of the cluster with the largest centroid

\end{algorithmic}
\end{algorithm}

The goal of the attack is to infer the target client's sensitive label information based on the training data.
For example, when training FL on medical image data, such as image data on breast cancer, the label of the cancer is very sensitive, and participants may not want to reveal this information.
A similar attack goal was considered in \cite{wainakh2022user, fu2022label}.
Our designed attack is based on the intuition that the top-$k$ indices of the locally converged model parameters are correlated with the labels of the local training data.
We train a classifier that accepts the observed index information as the input by supervised learning using a public test dataset and the output is the sensitive label set.
Access to the dataset is justified, for example, by the need for model validation, as described in Section \ref{sec:problem_scenario} and in previous studies on inference attacks \cite{wang2019beyond, hu2022federated}.
We design two basic methods---the Jaccard similarity-based nearest neighbor approach (\textsc{Jac}) and a neural network (\textsc{NN}).
The detailed algorithm is presented in Algorithm \ref{alg:attack_on_index}.
An overview of these methods is provided below:
\begin{enumerate}
  \setlength{\leftskip}{-1.2em}
  \item First, the server prepares the test data $X_l$ with label $l$ for all $l \in L$, where $L$ denotes the set of all possible labels.
  
  \item In each round $t$ $(\in T)$, an untrusted server observes the memory access patterns through side-channels, obtains the index information of the top-$k$ gradient indices $\mathbf{index}[i,t]$ for each user $i$, and stores it (lines 4--8).

  \item The server computes the gradient of the global model with $\theta^t$ and $X_l$, without model updates for each round $t$ $(\in T)$, using the test data categorized by labels, and obtains the top-$k$ indices $\mathbf{teacher}[l, t]$ as teacher data for each label (lines 9--12).
  
  \item After the completion of all rounds $T$, in \textsc{Jac}, we calculate the Jaccard similarity between observed access patterns, $\mathbin\Vert_{\tau \in T_i}{\mathbf{index}[i, \tau]}$ and $\mathbin\Vert_{\tau \in T_i}{\mathbf{teacher}[l, \tau]}$, for each label $l$ (lines 15--17). 
  Jaccard similarity is selected because, in the worst-case scenario, the index information transmitted by a participant is randomly shuffled, rendering the sequence meaningless.

  \item In \textsc{NN}, the attacker trains neural networks using $\mathbf{teacher}[l, t]$, with indices as the features and labels as the target (line 19).
  The outputs of the model are the scores of the label.
  Subsequently, we use a trained model to predict the labels included in the training data corresponding to the input, $\mathbf{index}[i]$.
  For this task, we design the two following NN-based methods.
  In the first method, a model, $M_t$, is trained during each round, $t$, and the output scores of the models are averaged to predict the labels (\textsc{NN}).
  In the second method, a single model, $M_{0}$, is trained using the concatenated indices of the entire round as input and a single output is obtained (\textsc{NN-single}).
  In our experiment, both cases involve a multilayer perceptron with three layers (described in Appendix F).
  Note that as the model input, index information is represented as a multi-hot vector.
  In the case of \textsc{NN-single}, each client participates in only a proportion of the rounds---the indices of the rounds they do not participate in are set to zero as the input to the model.
  Although \textsc{NN-single} is expected to be able to capture the correlation over rounds better than \textsc{NN}, this zeroization may reduce the accuracy.
  Finally, as in \textsc{Jac}, we store the scores for each label obtained via model prediction (lines 20--21).
  
  \item If the number of labels of the target client is known, the scores are sorted in descending order and the highest labels are returned.
  If the number of labels is unknown, K-means clustering is applied to the scores to classify them into 2 classes, and the labels with the highest centroid are returned (lines 23--24).
\end{enumerate}

Finally, the information obtained from the side-channels can also be used to design attacks for other purposes, such as additional features in reconstruction \cite{hitaj2017deep} or other inference attacks \cite{nasr2019comprehensive}.
The aim of this study is simply to demonstrate that the top-$k$ gradient indices that can be observed on untrusted servers contain sufficient information to cause privacy leakages; therefore, we leave the study of attacks for different purposes to future research.

\begin{figure*}[t]
    \centering
    \includegraphics[width=0.95\hsize]{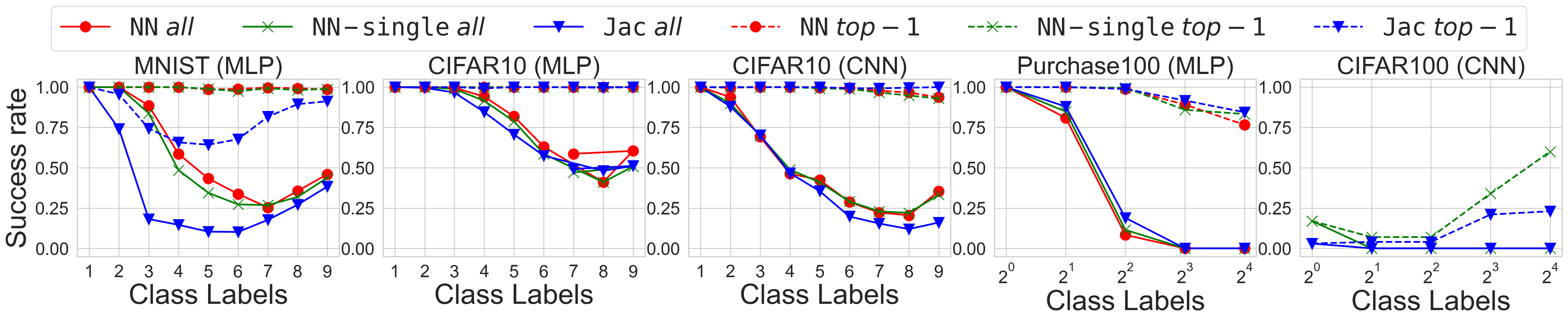}
    % \vspace{-3px}
    \caption{Attack results on datasets with a fixed number of labels: Vulnerable, especially when there are few labels.}
    \label{fig:attack_fixed_label}
\end{figure*}

\begin{figure*}[t]
    \centering
    \includegraphics[width=0.95\hsize]{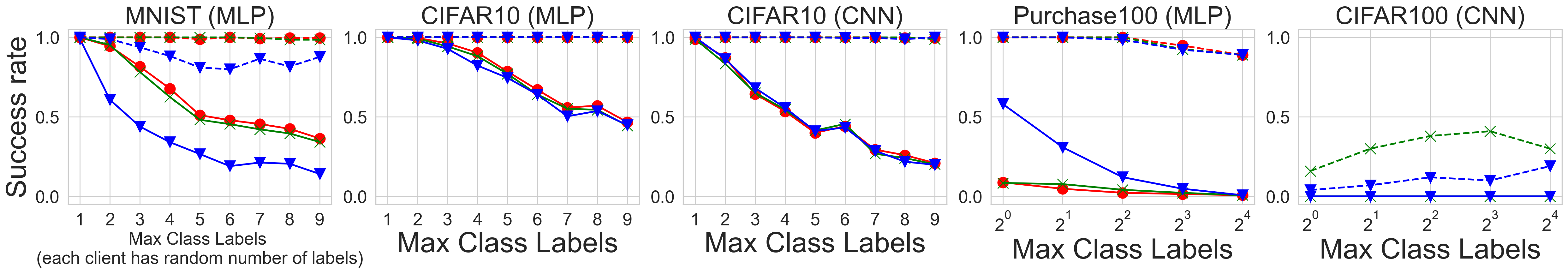}
    % \vspace{-3px}
    \caption{Attack results on datasets with a random number of labels (more difficult setting): When the number of labels is low, the attacker can attack the client without knowing the exact number of labels.}
    \label{fig:attack_random_label}
\end{figure*}

\subsection{Evaluation Task}
In our evaluation of attacks, the server performs an inference attack on any client in the scenario detailed in Section \ref{sec:problem_scenario}.
The clients have a subset of labels, and the attacker's goal is to infer the sensitive label set of a target client based on their training data.
The attacker selects any subset or the entire set of users and performs an inference attack on each user.
We utilize \textit{all} and \textit{top-1} as accuracy metrics for evaluating attack performance.
We define \textit{all} as the percentage of clients that match the inferred labels exactly, e.g., the inferred label set is \{1,3,5\}, and the target client's label set is \{1,3,5\}.
We define \textit{top-1} as the percentage of clients that contain the highest scored inferred label, e.g., the highest scored inferred label is five, and the target client's label set is \{4,5\}, which we consider to be a minimal privacy leak.
In addition, we adjust the distribution of the label set such that the client is able to control the difficulty of the attack.
The number of labels in the set and the number of labels that are \textit{fixed} or \textit{random} are configurable.
In the case of a \textit{fixed} label, all users exhibit the same number of labels, which is known to the attacker.
In the case of the \textit{random} label, the maximum number is assigned, and all users exhibit various numbers of labels.
Generally, \textit{random} label and larger numbers of labels are more difficult to infer.%, although the performance depends on the metrics, \textit{all} and \textit{top-1}.

\subsection{Empirical Analysis}
\label{sec:attack_experiment}
Here, we demonstrate the effectiveness of the designed attack.

\begin{table}[]
\caption{Datasets and global models in the experiments.}
\begin{center}
\renewcommand{\arraystretch}{1.2}
\begin{tabular}{lccc}
\hline
Dataset & Model (\#Params) & \#Label & \#Record (Test) \\ \hline \hline
\textbf{MNIST}  & MLP  (50890)  & 10     & 70000 (10000)\\ \hline
\multirow{2}{*}{\textbf{CIFAR10}} & \multirow{2}{*}{\begin{tabular}[c]{@{}c@{}}MLP (197320) \\ CNN (62006)\end{tabular}} & \multirow{2}{*}{10} & \multirow{2}{*}{60000 (10000)} \\
&  &   & \\ \hline
\textbf{Purchase100} & MLP (44964)               & 100    & 144000 (24000)\\ \hline
\textbf{CIFAR100}    & CNN (201588)              & 100    & 60000 (10000)\\ \hline
\end{tabular}
\label{table:dataset}
\end{center}
\vspace{-10px}
\end{table}

\noindent
\textbf{Setup.} Table \ref{table:dataset} lists the datasets and global models used in the experiments.
Details of the model, including the attacker's NN, are provided in Appendix F.
In addition to the well-known image datasets, MNIST and CIFAR 10 and 100, we also use Purchase100, which comprises tabular data used in \cite{eval2019usenix} for membership inference attacks.
We train the global models using different numbers of parameters, as listed in Table \ref{table:dataset}.
The learning algorithm is based on Algorithm \ref{alg:olive}, in which we provide the sparse ratio, $\alpha$, instead of $k$ in top-$k$.
FL's learning parameters include the number of users, $N$; the participant sampling rate, $q$; the number of rounds, $T$.
The default values are given by $(N, q, T, \alpha)=(1000, 0.1, 3, 0.1)$.
The attack methods are evaluated for \textsc{Jac}, \textsc{NN}, and \textsc{NN-single}, as described in the previous section.
$T$ is smaller than that in normal FL scenarios, which implies that our method requires only a few rounds of attacks.
All experimental source codes and datasets are open\footnote{\label{note:source}\url{https://github.com/FumiyukiKato/FL-TEE}}.

\begin{figure}[t]
    \centering
    \includegraphics[width=0.9\hsize]{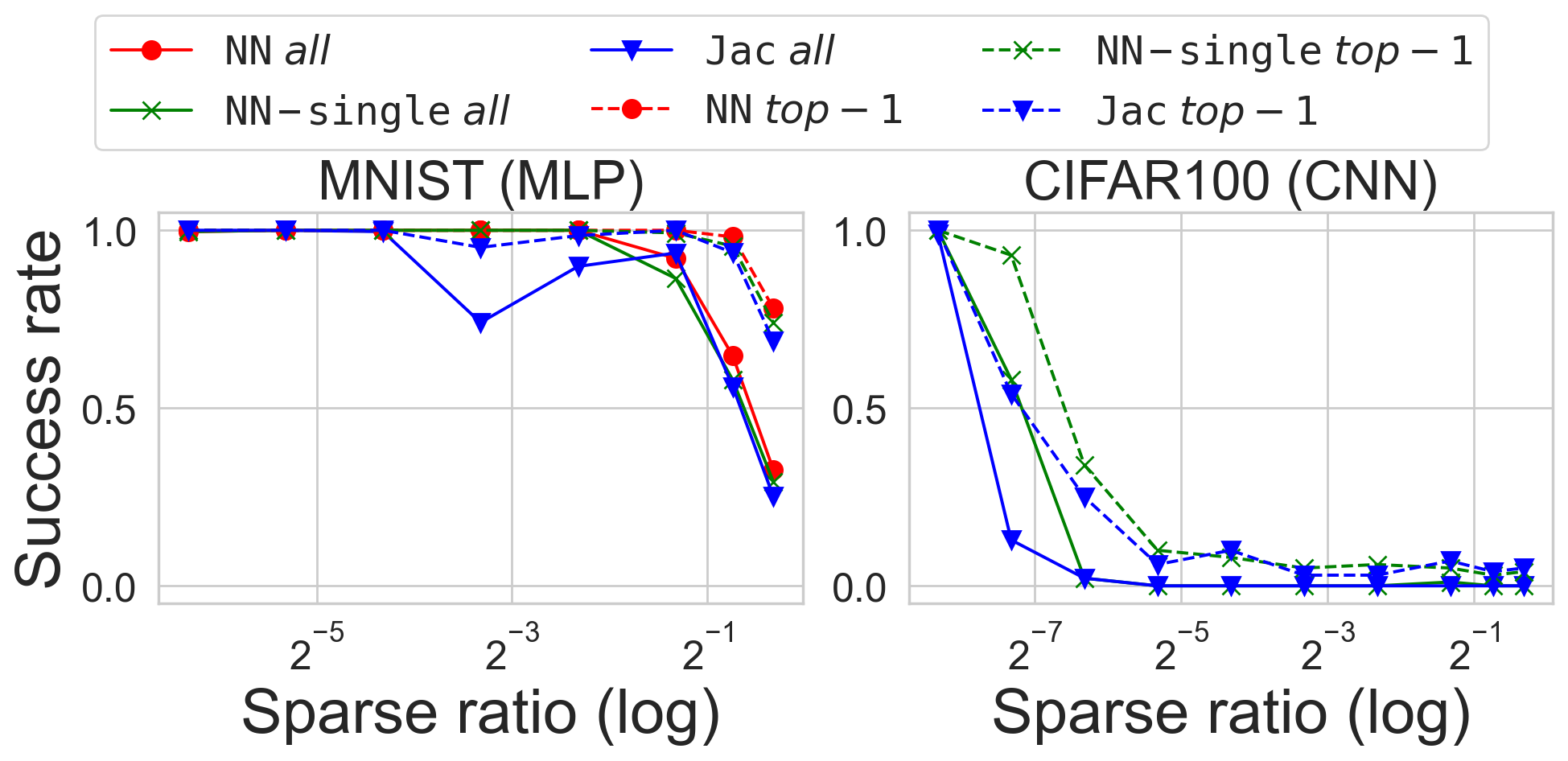}
    \caption{Attack results w.r.t. sparse ratios: Higher the sparsity, the more successful the attack tends to be.}
    \label{fig:attack_sparse_ratio}
\end{figure}

\begin{figure}[t]
    \centering
    \includegraphics[width=0.9\hsize]{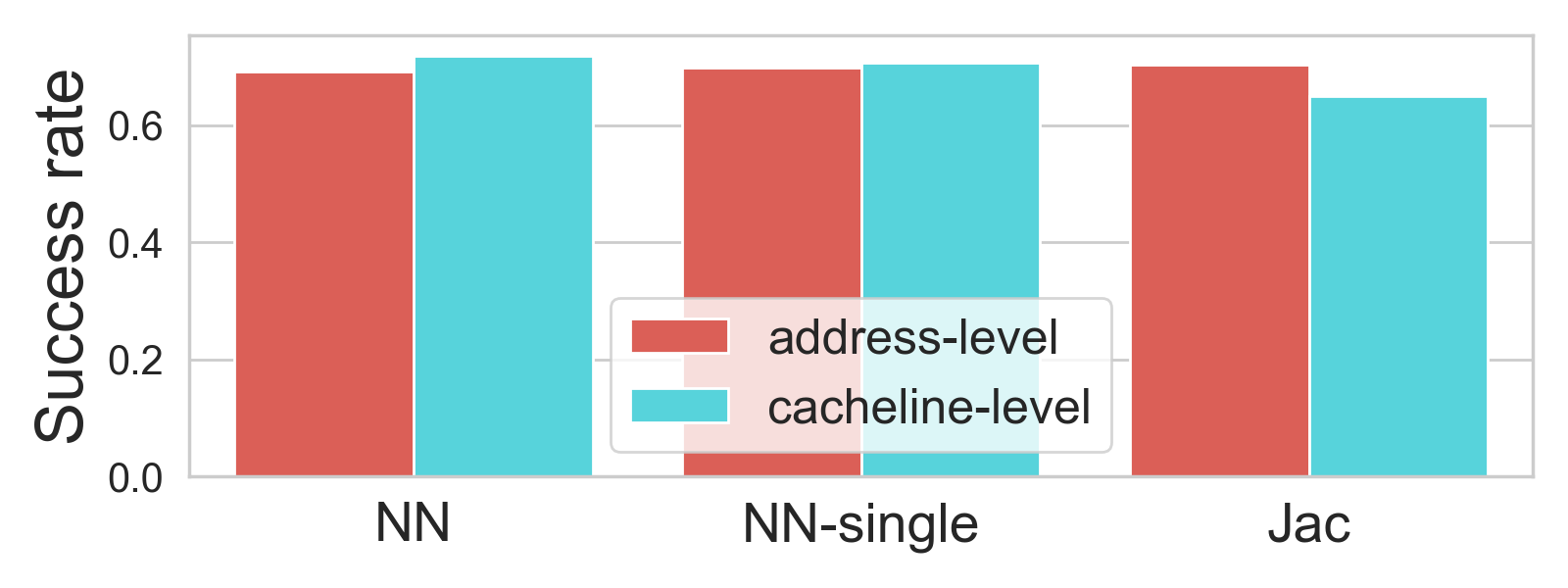}
    \caption{Cacheline-level leakage on CNN of CIFAR10: Attacks are possible with at least slightly less accuracy.}
    \label{fig:attack_cacheline_protection}
\end{figure}

\begin{figure}[t]
    \centering
    \includegraphics[width=0.8\hsize]{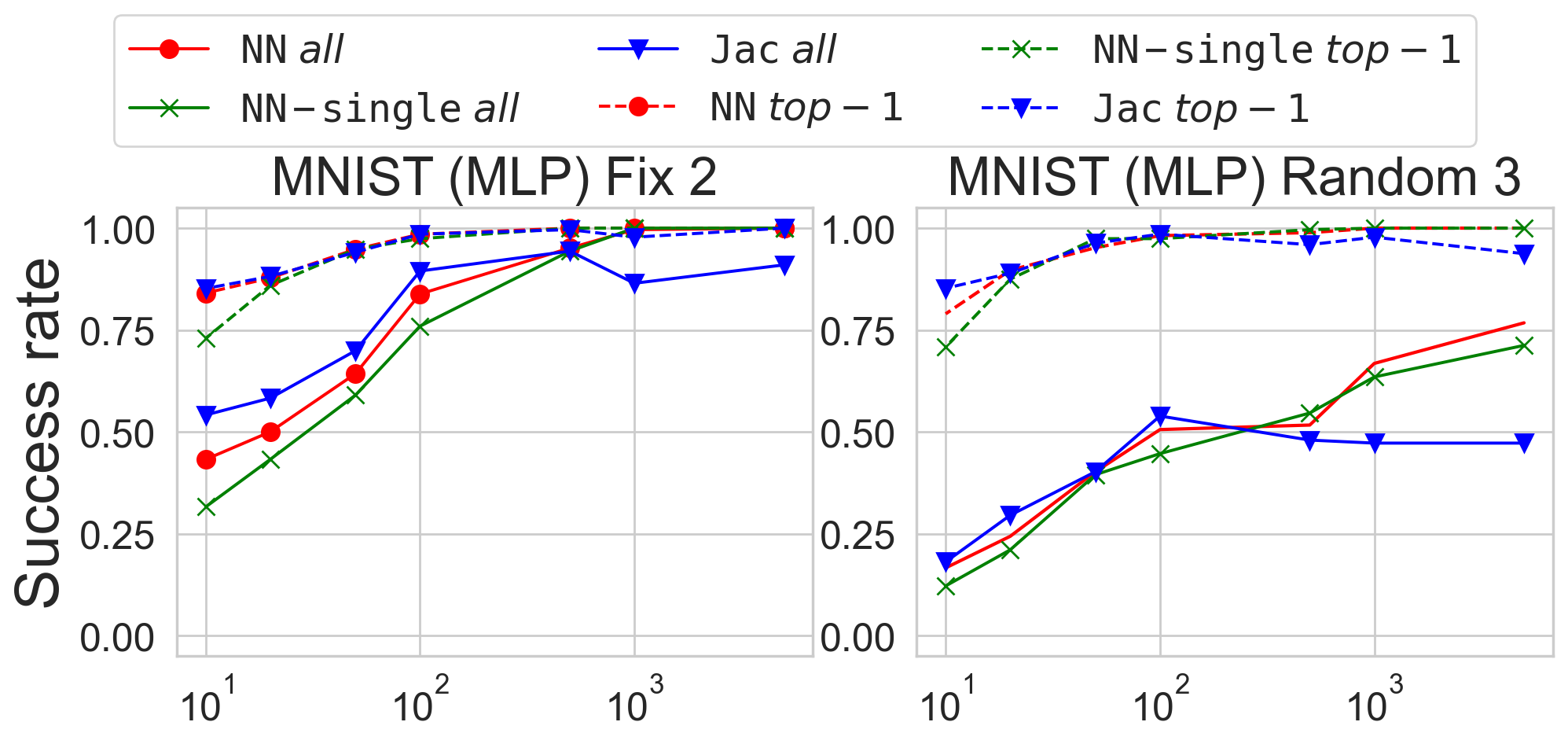}
    \includegraphics[width=0.85\hsize]{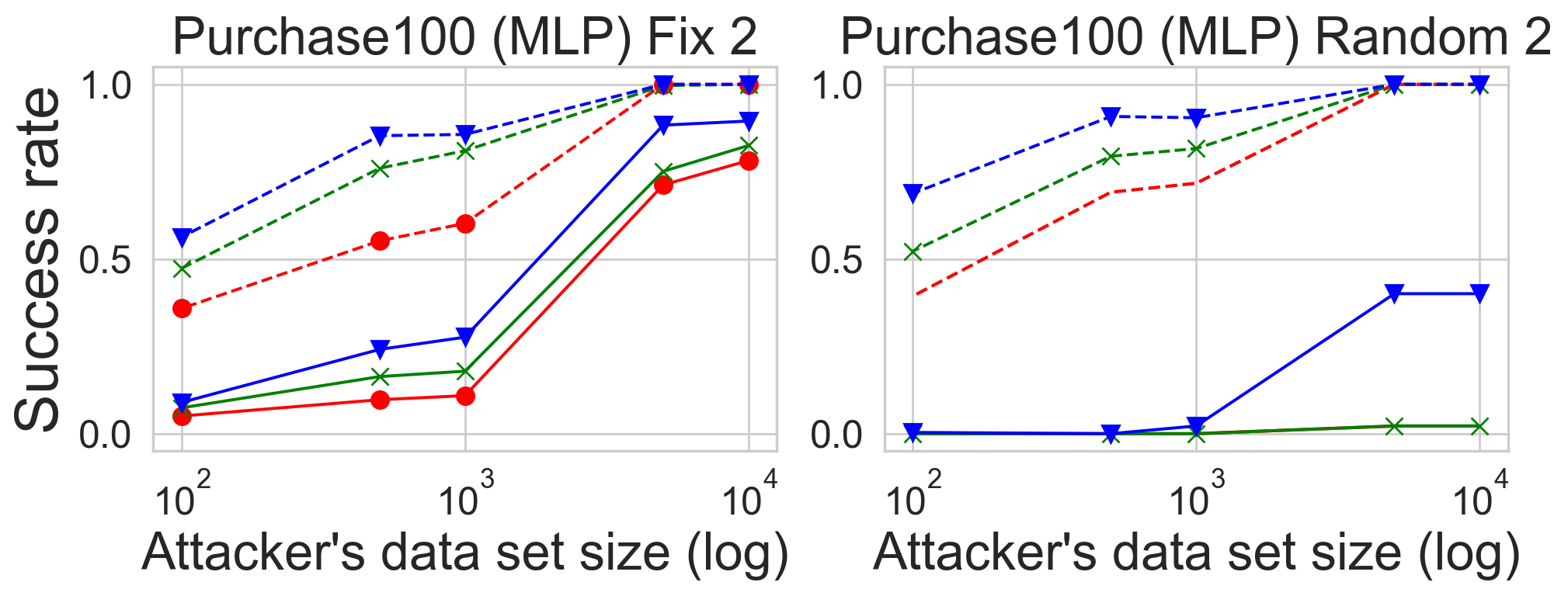}
    \caption{The size of data that an attacker needs to access to achieve high success rate can be very small.}
    \label{fig:attack_reducing_test_dataset}
\end{figure}

\noindent
\textbf{Results.}
Figure \ref{fig:attack_fixed_label} depicts the attack results for \textsc{NN}, \textsc{NN-single}, and \textsc{Jac} on all datasets with a \textit{fixed} number of labels, and Figure \ref{fig:attack_random_label} presents the results with a \textit{random} number of labels.
In CIFAR100, $T=1$ is used because the model size is large. %, and training the attack model over multiple rounds in \textsc{NN-single} is too time-consuming.
The y-axis represents the success rate of the attacks, and the x-axis represents the number of labels possessed by each client.
When the number of labels is small, all three attacks exhibit a high probability of success.
The success rate of \textit{top-1} is high irrespective of the number of labels, whereas \textit{all} decreases with each additional label. 
On CIFAR10, the MLP model maintains a higher success rate for a large number of labels compared to the CNN model.
This indicates that the complexity of the target model is directly related to the contribution of the index information to the attack.
The \textsc{NN}-based method is more powerful on MNIST, but it performs similarly to the other methods on the other datasets.
This indicates that the gradient index information is not complex and can be attacked using simple methods, such as \textsc{Jac}.
The results of \textsc{NN} and \textsc{NN-single} are almost identical; therefore, there is not much effective correlation across the rounds.
When the number of class label is 100 (Purchase100, CIFAR100), the success rate of the attack is reduced.
In particular, the accuracy of CIFAR100 is low in this case.
However, as shown in later, this is surprisingly improved by using a smaller sparse rate.
% Overall, these results here are sufficient to highlight our proposed threat.

Figure \ref{fig:attack_sparse_ratio} depicts the relationship between the sparse ratio and attack performance.
The number of client labels is fixed to two.
The results indicate that the sparse ratio is inversely related to the success rate of the attack.
This is because the indices of label-correlated gradients become more distinguishable as the sparsity increases.
In particular, the case of CIFAR100 demonstrates that the attack is successful only when the sparsity ratio is low.
For instance, when the sparsity ratio is 0.3\%, the success rate is almost 1.0.
Thus, sparsity ratio is an important factor in an attack.

Figure \ref{fig:attack_cacheline_protection} depicts a comparison of attack performance based only on index information observed at the cacheline granularity (64 B), which can be easily observed against SGX \cite{xu2015controlled} with CIFAR10 and CNN.
The accuracies are almost identical.
The \textsc{NN}-based method exhibits slightly higher accuracy, whereas \textsc{Jac} exhibits slightly poorer accuracy.
Therefore, the attack is still possible despite observations at the granularity of the cacheline, which indicates that the well-known vulnerability of SGX is sufficient to complete an attack.

Figure \ref{fig:attack_reducing_test_dataset} depicts the evaluation of the size of a dataset required by an attacker to succeed in an attack.
The default test dataset accessible to the attacker is presented in Table \ref{table:dataset}---we randomly reduce it on this basis while maintaining the same number of samples for each label.
We evaluate the number of labels in the fixed and random labels using the MNIST and Purchase100 datasets, respectively.
In MNIST, performance can be preserved even when the amount of data is reduced, which weakens the assumption on dataset size.
For example, it is surprisingly noted that, even with 100 samples (i.e., 10 samples per label and 1\% of the original evaluation), performance is not affected significantly.
On Purchase100, the impact is small, but a meaningful attack is possible with some reduction in data size.

\section{Oblivious Algorithms}
\label{sec:defense}
In this section, we focus on an aggregation algorithm that can cause privacy leakage, as described in the previous section, and discuss potential avenues of attack prevention.
The notation used here is identical to that in Section \ref{sec:security_analysis}.

First, we introduce the general ORAM-based method.
We initialize ORAM with $d$ zero values for the aggregated parameters, $g^*$; update the values with the received $nk$ gradients, $g$, sequentially; and finally retrieve the $d$ values from the ORAM.
Because ORAM completely hides memory access to $g^*$, the algorithm is fully oblivious.
However, as established in the experimental section, even the state-of-the-art PathORAM adapted to TEE \cite{DBLP:conf/ndss/SasyGF18} incurs a significant overhead---thus, a task-specific algorithm is preferable.

\subsection{Baseline method}
\label{sec:baseline}

Full obliviousness can be simply achieved by accessing all memory addresses to hide access to a specific address.
When accessing \textbf{$\mathbf{G^*}[i]$}, a dummy access is performed on \textbf{$\mathbf{G^*}[j]$} for each $j \in [d]$.
For each access, either a dummy or an updated true value is written, and the timing of writing the true value is hidden by an oblivious move (\hyperref[list:o_mov]{\texttt{o\_mov}}).
The Baseline algorithm is described in Algorithm \ref{alg:baseline}.
It accepts the concatenated gradients transmitted by all participants, $g$ ($nk$-dimensional vector), as input and returns the aggregated gradients, $g^*$ ($d$-dimensional vector) as output.
We make linear accesses to $\mathbf{G^*}$ for a number of times equal to the length of $\mathbf{G}$.
Assuming that the memory address is observable at the granularity of the cacheline, as in a traditional attack against the SGX \cite{xu2015controlled}, some optimization may be performed.
When the weight is four bytes (32-bit floating point) and cacheline is 64 bytes, a 16$\times$ acceleration can be achieved.
Irrespective of this optimization, the computational and spatial complexities are $O(nkd)$ and $O(nk+d)$, respectively.

\begin{proposition}
\label{prop:baseline}
Algorithm \ref{alg:baseline} is (cacheline-level) fully oblivious.
{\normalfont (A formal proof is provided in Appendix C.)}
\end{proposition}

\begin{algorithm}[t]
\small
\caption{Baseline}
\label{alg:baseline}
\begin{algorithmic}[1]
\renewcommand{\algorithmicrequire}{\textbf{Input:}}
\renewcommand{\algorithmicensure}{\textbf{Output:}}

\Require $g=g_1 \mathbin\Vert...\mathbin\Vert g_n$: concatenated gradients, $nk$ length
\Ensure $g^*$: aggregated parameters, $d$ length
\State initialize aggregated gradients $g^*$
\For{each ($idx$, $val$) $\in g$}
    \State /* $c$ is the number of weights included in one cacheline */
    \State /* of\textcompwordmark fset indicates the position of $idx$ in the cacheline */
    \For{each $(idx^*, val^*)$ $\in g^*$ if $idx^*\equiv$ of\textcompwordmark fset $(\mathrm{mod}\; c)$}
        \State $flag$ $\leftarrow$ $idx^* == idx$ \Comment{target index or not}
        \State $val^{\prime}$ $\leftarrow$ \texttt{o\_mov}$(flag, val^*, val^*+val)$ % \Comment{see Listing \ref{list:o_mov}}
        \State write $val^{\prime}$ into $idx^*$ of $g^*$
    \EndFor
\EndFor
\State \textbf{return} $g^*$
\end{algorithmic}
\end{algorithm}

\subsection{Advanced method}
\label{sec:advanced}

\begin{algorithm}[t]
\small
\caption{Advanced}
\label{alg:advanced}
\begin{algorithmic}[1]
\renewcommand{\algorithmicrequire}{\textbf{Input:}}
\renewcommand{\algorithmicensure}{\textbf{Output:}}

\Require $g=g_1 \mathbin\Vert...\mathbin\Vert g_n$: concatenated gradients, $nk$ length
\Ensure $g^*$: aggregated parameters, $d$ length
\State /* \textit{initialization}: prepare zero-valued gradients for each index */
\State $g'$ $\leftarrow$ $\{(1, 0),...,(d, 0)\}$ \Comment{all \texttt{value}s are zero}
\State $g$ $\leftarrow$ $g \mathbin\Vert g'$ \Comment{concatenation}
\State /* \textit{oblivious sort} in $O((nk+d)\log^2{(nk+d)})$ */
\State oblivious sort $g$ by \texttt{index}
\State /* \textit{oblivious folding} in $O(nk+d)$ */
\State $idx$ $\leftarrow$ \texttt{index} of the first weight of $g$
\State $val$ $\leftarrow$ \texttt{value} of the first weight of $g$
\For{each ($idx'$, $val'$) $\in g$} \Comment{Note: start from the second weight of $g$}
        \State $flag$ $\leftarrow$ $idx' == idx$
        \State /* $M_0$ is a dummy index and very large integer */
        \State $idx_{prior}, val_{prior}$ $\leftarrow$ \texttt{o\_mov}$(flag, (idx, val), (M_0, 0))$
        \State write $(idx_{prior}, val_{prior})$ into $idx'$ - 1 of $g$
        \State $idx, val$ $\leftarrow$ \texttt{o\_mov}$(flag, (idx',val'), (idx, val+val'))$
\EndFor
\State /* \textit{oblivious sort} in $O((nk+d)\log^2{(nk+d)})$ */
\State oblivious sort $g$ by \texttt{index} again

\State \textbf{return} take the first $d$ \texttt{value}s as $g^*$

\end{algorithmic}
\end{algorithm}

Here, we present a more advanced approach to FL aggregation.
In cases with large numbers of model parameters, $k$ and $d$ are significant factors and the computational complexity of the Baseline method becomes extremely high because of the product of $k$ and $d$.
As described in Algorithm \ref{alg:advanced}, we design a more efficient \textit{Advanced} algorithm by carefully analyzing the operations on the gradients.
Intuitively, our method is designed to compute $g^*$ directly from the operations on the gradient data, $g$, to eliminate access to each memory address of the aggregated gradients, $g^*$.
This avoids the overhead incurred by dummy access to $g^*$, as in the Baseline.
The method is divided into four main steps:
\textit{initialization} on gradients vector $g$ (line 1), oblivious sort (line 4), \textit{oblivious folding} (line 6), and a second oblivious sort (line 16).
For oblivious sort, we use Batcher’s Bitonic Sort \cite{batcher1968sorting}, which is implemented in a register-level oblivious manner using oblivious swap (\hyperref[list:o_swap]{\texttt{o\_swap}}) to compare and swap at all comparators in the sorting network obliviously.
Appendix E illustrates a running example for better understanding.

As given by Algorithm \ref{alg:advanced}, we first apply an initialization to $g$, where we prepare zero-valued gradients for each index between $1$ and $d$ (declared $g'$) and concatenate them with $g$ (lines 1--3).
Thus, $g$ has length $nk + d$.
This process guarantees that $g$ has at least one weight indexed for each value between $1$ and $d$; however, aggregation of the concatenated $g$ yields exactly the same result as the original $g$ because the added values are all zero.
We then apply an oblivious sort to $g$ using the parameter's index (lines 4--5).
Rather than eliminating the connection between the client and gradient, this serves as a preparation for subsequent operations to compute the per-index aggregate values.
Next, the \textit{oblivious folding} routine is executed (lines 6--14).
It linearly accesses the values of $g$ and cumulatively writes the sum of the values for each index in $g$.
Starting from the first place, it adds each value to the subsequent value if the neighboring indices are identical, and writes a zero-valued dummy index, $M_0$, in place of the original one.
$M_0$ is a large integer.
Otherwise, if the neighboring indices are different, we stop adding values, and the summation of the new index is initiated anew.
Thus, we finally obtain $g$ such that only the last weight of each index bears the correct index and aggregated value, and all the remaining ones bear dummy indices. 
In addition, the initialization process described above guarantees that $d$ distinct indices always exist.
In this phase, the index change-points on $g$ during folding are carefully hidden.
If the index change-points are exposed, the number corresponding to each index (i.e., the histogram of the indices) is leaked, which can cause catastrophic results.
Therefore, oblivious folding employs \hyperref[list:o_mov]{\texttt{o\_mov}} to make conditional updates oblivious and hide not only the memory access of the data, but also low-level instructions.
Finally, we apply an oblivious sort to $g$ (lines 15--16).
After sorting, in $g$, weights with indices between 1 and $d$ are arranged individually, followed by weights with dummy indices.
Finally, taking the \texttt{values} of the first $d$ weights of the sorted $g$, we return this as the final aggregated gradient, $g^*$ (line 17).

\begin{proposition}
Algorithm \ref{alg:advanced} is fully oblivious. 
\end{proposition}

\begin{proof}
The access pattern, $\mathbf{Accesses^{advanced}}$, is somewhat complicated, but obliviousness can be considered using a modular approach.
Our oblivious sort relies on Batcher’s Bitonic Sort, in which sorting is completed by comparing and swapping the data in a deterministic order, irrespective of the input data.
Therefore, access patterns generated using this method are always identical.
In oblivious folding, the gradient is linearly accessed once; thus, the generated access pattern is identical for all input data of equal length.
Finally, $\mathbf{Accesses^{advanced}}$ are identical and independent of inputs of equal length, this implies 0-statistical obliviousness.
\end{proof}

The complexity of the entire operation is $O((nk+d)\log^2{(nk+d)})$ in time and $O(nk+d)$ in space.
The proposed algorithm relies on an oblivious sort, which dominates the asymptotic computational complexity.
We use Batcher’s Bitonic Sort \cite{batcher1968sorting}, which has $O(n\log^2{n})$ time complexity.
The Advanced is asymptotically better than the Baseline because of the elimination of the $kd$ term.

\subsection{Optimization}
\label{sec:optimization}
In this subsection, we describe an optimization method that fits the basic SGX memory characteristics.
The current SGX comprises two major levels of memory size optimization.
The first factor is the size of the L3 cache (e.g., 8 MB).
In SGX, the acceleration is significant because the cache hit reduces not only the memory access time but also the data-decrypting process.
The second factor is the EPC size (e.g., 96 MB).
As mentioned in Section \ref{sec:tee}, accessing data outside the EPC incurs serious paging overhead.
Compared to the proposed methods, the Baseline is computationally expensive; however, most memory accesses are linear. 
Thus, it is greatly accelerated by the high cache hit rates and the prefetch functionality of the CPU.
However, in Advanced, the low locality of memory accesses in Batcher's sort reduces the cache and EPC hit rates.

Therefore, optimization is performed by introducing a function to split users into appropriate groups before executing Advanced to keep the data processed at one time within the EPC size.
This procedure involves the following steps: (1) divide into groups of $h$ users each; (2) aggregate values for each group using Advanced; (3) record the aggregated value in the enclave, and carry over the result to the next group; and (4) only average the result when all groups have been completed and then load them from the enclave to the untrusted area.
Note that the improvement to Advanced does not change its security characteristics.
An external attacker can only see the encrypted data, and any irregularities in the order or content of the grouped data can be detected and aborted by enclave.
The key parameter is the number of people, $h$, in each group.
The overall computational complexity increases slightly to $O(n/h((hk+d)\log^2{(hk+d)}))$.
However, this hides the acceleration induced by cache hits and/or the overhead incurred by repeated data loading.
Basically, although lowering $h$ improves the benefit of cache hits, lowering it too much results in a large amount of data loading.
The optimal value of $h$ is independent of data and can be explored offline.
Our results indicate that there exists an optimal $h$ that achieves the highest efficiency in the experiment.

\subsection{Relaxation of Obliviousness}
\label{sec:extend_method}

We investigate further improvements by relaxing the condition of full obliviousness to achieve better efficiency.
A relaxed security definition that has recently garnered attention is that of \textit{ differentially oblivious} (DO) \cite{chan2019foundations, allen2019algorithmic, 10.1145/3243734.3243851, chu2021differentially, vldb2022do}.
% Following the definition of Section \ref{sec:memory_access_pattern_leak}, $(\epsilon, \delta)$-DO means that for any neighboring inputs $I, I'$, it holds that
% \begin{equation}
% \nonumber
%   \Pr[\mathbf{Accesses}^{M}(\lambda, I)] \le e^{\epsilon} \Pr[\mathbf{Accesses}^{M}(\lambda, I')] + \delta.
% \end{equation}
DO is DP applied to obliviousness.
This relaxation can theoretically improves the efficiency from full obliviousness.
In practice, improvements have been reported for RDB queries \cite{vldb2022do} whose security model, in which access pattern leakage within the enclave is out of the scope, differs from ours.

However, DO is unlikely to work in the FL setting.
DO approaches commonly guarantee DP for the histogram of observed memory accesses.
We construct a DO algorithm based on \cite{allen2019algorithmic, 10.1145/3243734.3243851}.
The procedure involves the following steps: pad dummy data, perform an obvious shuffle (or sorting), and update $g^*$ by performing linear access on \textbf{$\mathbf{G}$}.
The observed memory access pattern is equivalent to a histogram of the indices corresponding to all gradients, and the dummy data are required to be padded with sufficient random noise to make this histogram DP.
However, this inevitably incurs prohibitive costs in the FL setting.
% \footnote{We have confirmed this prohibitive cost in our preliminary experiments.}
The first reason for this is that the randomization mechanism can only be implemented by padding dummy data \cite{case2021privacy}, which implies that only positive noise can be added, and the algorithms covered by padding are limited (e.g., the shifted Laplace mechanism).
The second reason is critical in our case and differs from previous studies \cite{allen2019algorithmic, 10.1145/3243734.3243851}. 
Considering that the ML model dimension, $d$, and even the sparsified dimension, $k$, can be large, noise easily becomes significant.
For example, considering the DO guaranteed by Laplace noise, where $k$ denotes the sensitivity and $d$ is the dimension of the histogram, the amount of noise is proportional to $kd$ and multiplied by a non-negligible constant, owing to the first reason \cite{allen2019algorithmic}.
This produces huge array data to which oblivious operations must be applied, resulting in a larger overhead than in the fully oblivious case.

\subsection{Experimental results}
\label{sec:experiment}

In this section, we demonstrate the efficiency of the designed defense method on a practical scale.
Because it is obvious that the proposed algorithms provide complete defense against our attack method, their attack performances are not evaluated here.
In addition, our previous algorithms do not degrade utility---the only trade-off for enhanced security is computational efficiency.

\noindent
\textbf{Setup: } We use an HP Z2 SFF G4 Workstation with a Intel Xeon E-2174G CPU, 64 GB RAM, and 8 MB L3 cache, which supports the SGX instruction set and has 128 MB processor reserved memory, of which 96 MB EPC is available for user use.
We use the same datasets as those in Table \ref{table:dataset} and synthetic data.
Note that the proposed method is fully oblivious and its efficiency depends only on the model size.
The aggregation methods are the \textit{Non Oblivious} (\textit{linear} algorithm in Section \ref{sec:security_analysis}), the \textit{Baseline} (Algorithm \ref{alg:baseline}), the \textit{Advanced}(Algorithm \ref{alg:advanced}), and \textit{PathORAM}.
We implement PathORAM based on an open-source library\footnote{\url{https://github.com/mobilecoinofficial/mc-oblivious}} that involves a Rust implementation of Zerotrace \cite{DBLP:conf/ndss/SasyGF18}.
The stash size is fixed to 20.
In the experiments, we use \textit{execution time} as an efficiency metric.
We measure the time required by an untrusted server from loading the encrypted data to the enclave to completion of aggregation.

\begin{figure}[t]
    \centering
    \includegraphics[width=0.90\hsize]{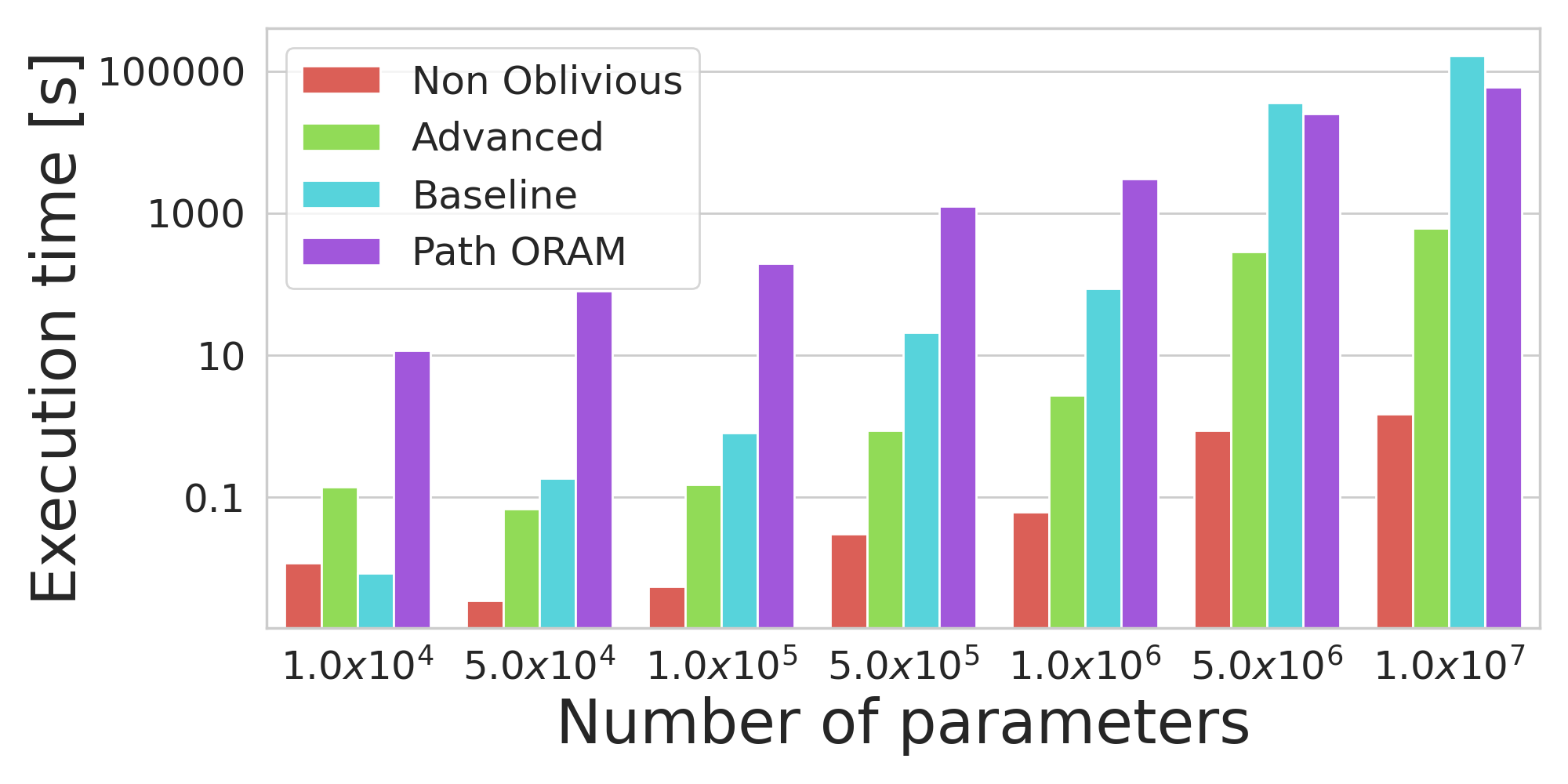}
    % \vspace{-5px}
    \caption{Performance results on a synthetic dataset w.r.t. models of various sizes: \textit{Advanced} functions efficiently. $\alpha$ (sparse ratio) $=0.01$ and $n$ (number of clients per round) $=100$.}
    \label{fig:performance_on_artificial_data}
\end{figure}

\noindent
\textbf{Results: } Figure \ref{fig:performance_on_artificial_data} depicts the execution time for the aggregation operation on the synthetic dataset with respect to model size.
$\alpha$ is fixed to 0.01, and the x-axis represents the original model parameter size, $d$.
The proposed \textit{Advanced} is approximately one order of magnitude faster than \textit{Baseline}.
Moreover, it is more robust with respect to an increase in the number of parameters.
Only when the number of parameters is very small is \textit{Baseline} faster than \textit{Advanced}, because when the model is extremely small, \textit{Baseline}'s simplicity becomes dominant.
\textit{PathORAM} also incurs a large overhead.
The theoretical asymptotic complexity of the original PathORAM-based algorithm is $O((nk)\log{(d)})$ because a single update on ORAM can be performed in $O(\log{(d)})$.
However, this is an ideal case and the overhead of the constant factor is large when PathORAM is adapted to the SGX security model (i.e., ZeroTrace \cite{DBLP:conf/ndss/SasyGF18}).
The overhead is primarily induced by the \textit{refresh} operation corresponding to each update and the oblivious reading of the position maps.
The result suggests that \textit{PathORAM}'s superiority does not appear until the data size increases hugely.
Overall, the results indicate that the aggregation process can be completed in a few seconds, even if the model scale involves approximately 1M parameters. %, which is larger than any of the models listed in Table \ref{table:dataset}.
%, and our proposed method has practical performance on the real-world scale from parameter size of the global models in Table \ref{table:dataset}.

\begin{figure}[t]
    \centering
    \includegraphics[width=0.9\hsize]{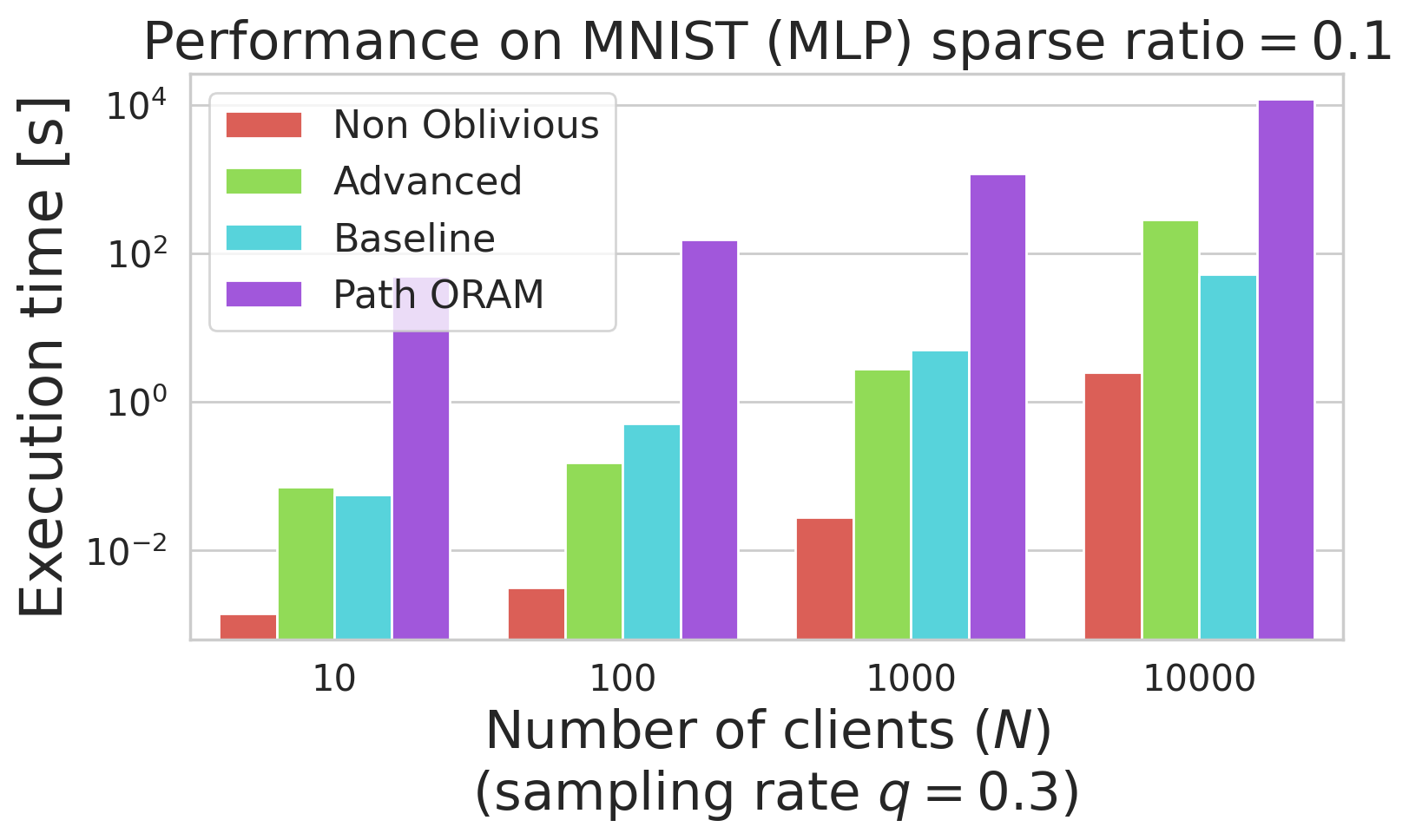}
    \caption{Performance results w.r.t. various numbers of clients ($N$) at low sparsity ($\alpha=0.1$): the \textit{Advanced} gradually worsens with increasing number of clients.}
    \label{fig:performance_on_mnist_user}
\end{figure}

\begin{figure}[t]
\begin{minipage}[t]{0.48\hsize}
    \centering
    \includegraphics[width=\hsize]{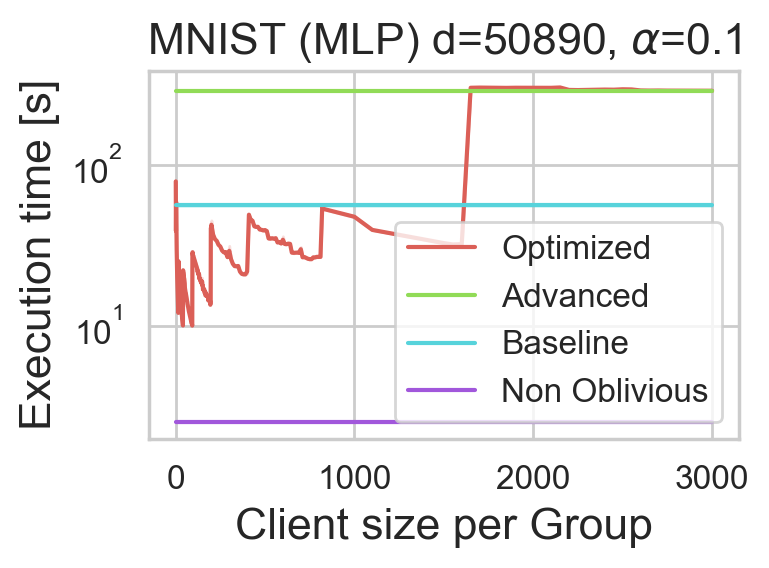}
\end{minipage}
\hfill
\begin{minipage}[t]{0.5\hsize}
    \centering
    \includegraphics[width=\hsize]{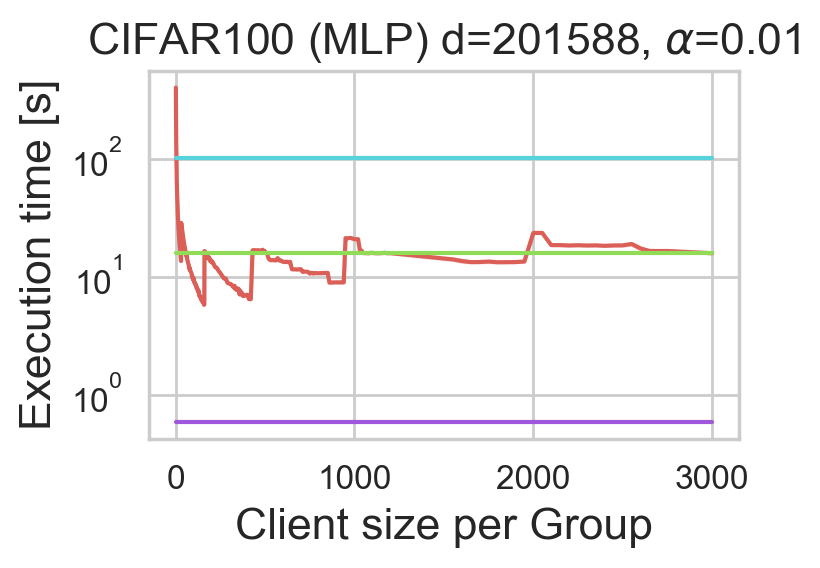}
\end{minipage}
\caption{The effects of optimizing the \textit{Advanced} on MLP models on MNIST (left) and CIFAR100 (right).}
    \label{fig:client_optimized_mnist}
\end{figure}

Figure \ref{fig:performance_on_mnist_user} depicts the performances on MNIST (MLP) corresponding to various numbers of clients and low sparsity ($\alpha=0.1$).
The \textit{Baseline} method is more efficient when the number of clients, $N$, is large ($10^4$).
Firstly, the model size $d$ is fairly small (i.e., MNIST (MLP) consists of only 50K parameters).
Hence, the overhead of the dummy access operations of \textit{Baseline} is not significant.
The second reason is that the lower sparsity and higher number of clients increases $nk$, which increases the overhead for both \textit{Baseline} and \textit{Advanced}, but affects \textit{Advanced} more, as explained by the analysis of cache hits in Section \ref{sec:optimization}.
At $N=10^4$, the memory size required by \textit{Advanced} is given by (vector to be obliviously sorted) = $5089*8*3000+50890*8 \approx $ 122 MB (> 96 MB of EPC size) since each cell of gradient is 8 bytes (32-bit unsigned integer for index and 32-bit floating point for value).
Batcher's sort requires repeated accesses between two very distant points on the vector, which could require a large number of pagings until \textit{Advanced} finishes; however, in \textit{Baseline}, this hardly occurs.
However, the optimization introduced in Section \ref{sec:optimization} successfully addresses this problem.

Figure \ref{fig:client_optimized_mnist} illustrates the effects of the optimization method on \textit{Advanced}.
The left figure shows the results under the same conditions as the rightmost bars in Figure \ref{fig:performance_on_mnist_user} ($N=10^4$), indicating that \textit{Advanced} is dramatically faster with an optimal client size.
When the number of clients per group, $h$ (represented along the x-axis), is small, the costs of iterative loading to the enclave become dominant, and the overhead conversely increases.
However, if $h$ is gradually increased, the execution time decreases.
Considering that the size of the L3 cache is 8 MB and data size per user is $d\alpha=0.04$ MB, the L3 cache can accommodate up to approximately 200 clients.
The results of MNIST (MLP) indicate that the lowest is, approximately 10 s, at around $h=100$, which is a significant improvement compared to 290 s in the original \textit{Advanced}.
The small waviness of the plot appears to be related to the L2 cache (1 MB), which does not have an impact as large as that of the L3 cache.
The efficiency decreases significantly around $h=2000$, owing to the EPC paging.
The figure on the right depicts the results on CIFAR100 (MLP) at $\alpha=0.01$ and $N=10^4$.
In this case, \textit{Advanced} is initially much faster, but there is an optimal $h$ that can be further improved.
The pre-optimization execution time of 16 s is reduced to 5.7 s at around 150 clients.

\subsection{Discussion}
\label{sec:defense_discussion}
\textbf{Threat assumption.}
Boenisch et al. \cite{boenisch2021curious} reported that \textit{malicious} servers improve inference attack performance beyond \textit{semi-honest}.
This type of attack involves crafting global model parameters (called \textit{trap weights} in \cite{boenisch2021curious}) and controlling client selection in rounds to highlight the updates of the target user by a malicious server.
To prevent parameter tampering, \cite{bottoni2022verifiable} proposed a defense strategy using a cryptographic commitment scheme.
The \method can adopt a similar strategy based on a cryptographic signature.
Aggregation is performed within the enclave, and the aggregated global model is signed with the private key in the enclave.
This ensures that the model is not tampered with outside the enclave, i.e., malicious server.
Any client can verify this using a public key which can be easily distributed after RA.
In addition, TEE prevents malicious client selection by securely running in the enclave.
Therefore, privacy is not violated at least such type of the attack.
Other possible malicious server behaviors can influence the security of the \method, including denial-of-service (DoS) attacks \cite{jang2017sgx}, which are outside the threat model of the \method, as well as TEE and are difficult to prevent.

\textbf{Security of SGX.}
Finally, we discuss the use of SGX as a security primitive against known attacks.
According to \cite{nilsson2020survey}, the objectives of attacks against SGX can be classified into the following three: (1) stealing memory/page access patterns or instruction traces \cite{xu2015controlled, van2017telling, brasser2017software, lee2017inferring}, (2) reading out memory content \cite{chen2019sgxpectre, van2018foreshadow}, and (3) fault injection \cite{murdock2020plundervolt}.
(1) is the target of our defense.
The speculative execution attacks of (2) are mostly handled by microcode patches.
Hence, the protection is usually not required in the application.
However, if the microcode is not updated, the gradient information of the enclave may be stolen by a malicious attacker, which is beyond the scope of this study.
The fault injection of (3) is covered within the scope of microcode/hardware \cite{nilsson2020survey, murdock2020plundervolt} and lies outside our security. 
This may cause DoS even using TEE \cite{jang2017sgx}.

In addition, another risk exists if malicious code is embedded in the code executed in the enclave.
This can be prevented by verifying the enclave state using RA; however, this requires the source code to be publicly available and assessed.
Further, as discussed in \cite{van2019tale}, the SDK may involve unintended vulnerabilities.
To benefit from the security of SGX, the code of TCB must be written properly.

\section{Related works}
\label{sec:related_works}

\textbf{Security and Privacy threats in FL.}
FL contains many attack surfaces because of its decentralized and collaborative scheme.
These can be broadly classified into inference attacks by semi-honest parties \cite{nasr2019comprehensive, wainakh2022user, fu2022label} and attacks that degrade or control the quality of the model by malicious parties \cite{10.1145/3322205.3311083, zhao2021sear, pmlr-v108-bagdasaryan20a}.
However, \cite{boenisch2021curious} demonstrated that malicious servers may enable effective inference attacks by crafting aggregated parameters.
Our target is taken to be an inference attack by a semi-honest server.
Inference attacks include reconstruction \cite{hitaj2017deep, bhowmick2018protection}, membership \cite{nasr2019comprehensive}, and label inferences \cite{wainakh2022user, fu2022label}.
In particular, it has been reported that shared parameters observed by a server contain large amounts of private information \cite{zhu2020deep, zhao2020idlg}.
Our work targets gradient-based label inference attacks, \cite{wainakh2022user, fu2022label} use the gradients themselves, focusing on the values, and not only on the indices leaking from the side-channel, as in our method.
To the best of our knowledge, this is the first study to demonstrate label inference using only sparsified index information.

Secure aggregation (SA) \cite{mohamad2023sok} is a popular FL method for concealing individual parameters from the server and it is based on the lightweight pairwise-masking method \cite{DBLP:conf/icml/KairouzL021, 10.1145/3133956.3133982, ergun2021sparsified}, homomorphic encryption \cite{aono2017privacy, hao2019towards} or TEE \cite{zhang2021shufflefl, zhao2021sear}.
Another approach is to ensure (local) DP for the parameter to privatize the shared data; however, this sacrifices the utility of the model \cite{zhao2020privatedl, zhao2020local, ijcai2021-217}.
In this study, we study SA using TEE---further details are provided in the next paragraph.
Recent studies have investigated combinations of SA and sparsification, such as random-$k$ \cite{ergun2021sparsified} and top-$k$ \cite{lu2023top}.
However, these are not in harmony because they require the same sparsified indices among clients for mask cancellation.
\cite{lu2023top} proposed generation of common masks by taking a union set of top-$k$ indexes among clients, which incurs extra communication costs and strong constraints.
This can be serious for the top-$k$ because, in fact, Ergun et al. \cite{ergun2021sparsified} showed that the top-$k$ indices exhibits little overlap between clients, which is especially noticeable in the non-i.i.d. as in FL.
In \cite{ergun2021sparsified}, only a pair of users exhibited a common index; however, this was applicable only to random-$k$ sparsification.
In the case of TEE, a common index or random-$k$ is not required; but, individual indices can still be leaked through side-channels.
Therefore, our work focuses on attacks and defense strategies at this point.

\textbf{FL with TEE.} 
Using TEE in FL is a promising approach \cite{mo2021ppfl, zhang2021shufflefl, 10.1145/3477114.3488765, zhang2021citadel, flatee2021} in this context.
In addition to the confidentiality of gradients (i.e., SA functionality), TEE provides remote program integrity and verifiability via \textit{remote attestation}.
The major difference from centralized ML using TEE \cite{ohrimenko2016oblivious, hunt2018chiron} is that the training data are not shared to the server and they are not centralized in the latter case, which can be critical because of privacy or contractual/regulatory reasons or for practical reasons, i.e., big and fast data at multiple edges.
It is also important to outsource heavy computations required for ML training from TEE's limited computational resources to external clients.
PPFL \cite{mo2021ppfl} uses a TEE to hide parameters to prevent semi-honest client and server attacks on a global model.
Citadel \cite{zhang2021citadel} addressed the important goal of making the design of models confidential in collaborative ML using TEE.
However, side-channel attacks were not covered.
In \cite{zhang2021shufflefl} and \cite{10.1145/3477114.3488765}, the gradient aggregation step was taken to be hierarchical and/or partitioned using multiple servers such that the gradient information could only be partially observed by each server.
The authors assumed reconstruction attack and that a gradient leakage of less than 80\% was acceptable, which differs from our assumption completely.
In this study, the attack is based only on the gradient index information, and the goal is label inference.
Further, our proposed defense is more practical since we require only one server and one TEE, compared to the aforementioned method of distributed processing, which assumes multiple non-colluding servers with TEEs.
Flatee \cite{flatee2021} used TEE and DP in FL.
\cite{flatee2021} mentioned server-side obliviousness, but did not provide any analysis and solution for the leakages via side-channels.
Our study includes an analysis of access patterns in the aggregation procedure of FL and the design and demonstration of attack methods to motivate our defenses thoroughly in addition to specific solutions that lead to stronger security than any other method in FL on a single central TEE.

\textbf{Oblivious techniques.} 
The oblivious algorithm \cite{goldreich1987towards, pathoram2013ccs, ohrimenko2016oblivious} is known to induce only independent memory access patterns for the input data.
Although PathORAM \cite{pathoram2013ccs} is the most efficient ORAM implementation, it assumes a private memory space of a certain size (called as \textit{client storage}) and is not applicable to Intel SGX \cite{DBLP:conf/ndss/SasyGF18}.
Zerotrace \cite{DBLP:conf/ndss/SasyGF18} adapted PathORAM to the SGX security model, in which the register is only private memory.
The authors used the oblivious primitive proposed in \cite{ohrimenko2016oblivious}, in which the program did not leak instruction sequences from the CPU register, using x86 conditional instructions.
Our proposed algorithm also uses the low-level primitives; however, high-level algorithms are considerably different.
\cite{zheng2017opaque} studied oblivious SQL processing.
Their proposal included a \textit{group-by} query, which is similar to our proposed algorithm in concept.
Our aggregation algorithm computes the summed dense gradients based on multiple sparse gradients, which can be viewed as a special case of the \textit{group-by} query.
But, our method is more specialized, for instance, we first prepare the zero-initialized dense gradients to hide the all of index set that are included and then obliviously aggregated, which is impossible in the case of \textit{group-by}.
In addition, the aforementioned algorithms are fundamentally different because they focus on the data distributed across nodes.
Further, \cite{zheng2017opaque} did not consider the technique proposed by \cite{ohrimenko2016oblivious} for linear access, which can induce additional information leaks in the conditional code \cite{xu2015controlled}.
\cite{rane2015raccoon, sinha2017compiler} studied compiling and transforming approaches from high-level source code to low-level oblivious code.
They proposed a compiler that automatically identifies non-oblivious parts of the original source code and fixes them. But, the authors did not provide customized high-level algorithms for specific purposes, unlike our method.
The Differentially Obliviousness (DO) \cite{chan2019foundations, allen2019algorithmic, vldb2022do} is described in detail in Section \ref{sec:extend_method}.

\section{Conclusions}
\label{sec:conclusion}
In this study, we analyzed the risks of FL with server-side TEE in a sparsified gradient setting, and designed and demonstrated a novel inference attack using gradient index information that is observable from side-channels.
To mitigate these risks, we proposed an oblivious federated learning system, called the \method, by designing fully oblivious but efficient algorithms.
Our experimental results demonstrated that the proposed algorithm is more efficient than the state-of-the-art general-purpose ORAM and can serve as a practical method on a real-world scale.
% We also suggested several research directions for sparsified index information.
We believe that our study is useful for realizing privacy-preserving FL using a TEE.

\begin{acks}
 This work was supported by the Research Fund of JST CREST (No. JPMJCR21M2), JST SICORP (No. JPMJSC2107), and JSPS KAKENHI (21K19767, 22H03595).
\end{acks}

%\clearpage

\bibliographystyle{ACM-Reference-Format}
\bibliography{sample}

\clearpage
\appendix

\section{Oblivious primitives}
\label{appendix:o_primitives}
Here we describe the detailed implementation of the oblivious primitive we used.
The C inline assembler-like pseudo-code is shown here.
However, the Rust implementation we actually used is available in the public repository.

\begin{lstlisting}[basicstyle=\ttfamily\footnotesize, frame=single,caption={Oblivious move based on \texttt{CMOV} },captionpos=b,label={list:o_mov}]
int o_mov(bool flag, uint64 x, uint64 y) {
    /* inline assembly */
    /* register mapping: 
        flag  => ecx, x  => rdx, y  => r8  */
    mov rax, rdx
    test ecx, -1
    cmovz rax, r8
    return rax
}
\end{lstlisting}

\begin{lstlisting}[basicstyle=\ttfamily\footnotesize, frame=single,caption={Oblivious swap based on \texttt{CMOV} },captionpos=b,label={list:o_swap}]
int o_swap(bool flag, uint64 x, uint64 y) {
    /* inline assembly */
    /* register mapping: 
        flag  => rax, x  => rdx, y  => r8  */
    test rax, rax
    mov r10, r8
    mov r9, rdx
    mov r11, r9
    cmovnz r9, r10
    cmovnz r10, r11
    mov rdx r9
    mov r8, r10
}
\end{lstlisting}

\section{General FL Aggregation Algorithm}
\label{appendix:aggregation_alg}
We show a general FL aggregation algorithm.
The main focus here is on which memory addresses are accessed in the operation.

\begin{algorithm}[h!]
\small
\caption{Linear algorithm (and averaging and perturbing)}
\label{alg:linear}
\begin{algorithmic}[1]
\renewcommand{\algorithmicrequire}{\textbf{Input:}}
\renewcommand{\algorithmicensure}{\textbf{Output:}}

\Require $\mathbf{G} = \mathbf{G}_1 \mathbin\Vert ... \mathbin\Vert  \mathbf{G}_n$ where $\mathbf{G}_p$ $(p\in [n])$ is gradient from user $p$ and $k$ length vector, $\mathbf{G}$ is $nk$ length vector and $\mathbf{G}$'s element $g_q$ $(q\in [nk])$ is composed of (\texttt{index}, \texttt{value})
\Ensure $\mathbf{G^*}$: Aggregated gradient and $d$ length vector
\Procedure{Aggregation}{$\mathbf{G}$}
    \State \textbf{/* linear algorithm */}
    \State Initialize gradients $\mathbf{G^*}$
    \For{$i= 1,...,n$}
        \For{$j= 1,...,k$}
            \State $\mathbf{G^*}[\mathbf{G}[k*(i-1) + j].\texttt{index}] \mathrel{+}= \mathbf{G}[k*(i-1) + j].\texttt{value}$
        \EndFor
    \EndFor
    \State /* Averaging and Perturbing with linear access */
    \For{$i= 1,...,d$}
        \State $\mathbf{G^*}[i] \mathrel{/}= n$
    \EndFor
    \For{$i= 1,...,d$}
        \State $z \leftarrow$ Random noise (e.g., Gaussian distribution)
        \State $\mathbf{G^*}[i] \mathrel{+}= z$
    \EndFor
    \State \textbf{return} $\mathbf{G^*}$
\EndProcedure

\end{algorithmic}
\end{algorithm}

\section{Proofs of obliviousness}
\label{appendix:proofs_obliviousness}

Proof of Proposition \ref{prop:linear}.
\begin{proof}
Let the access pattern of \textit{linear} algorithm for dense gradients be $\mathbf{Accesses^{dense}}$; then, the pattern is represented as follows:
\begin{equation}
\nonumber
\small
\begin{split}
  \mathbf{Accesses^{dense}} &= \\
  [(\mathbf{G}[1], \mathtt{read}&,*), (\mathbf{G^*}[1], \mathtt{read},*), (\mathbf{G^*}[1], \mathtt{write},*),..., \\
  (\mathbf{G}[nd], \mathtt{read}&,*), (\mathbf{G^*}[d], \mathtt{read},*), (\mathbf{G^*}[d], \mathtt{write},*)]
 \end{split}
\end{equation}
\noindent
This means reading the sent gradients \textbf{$\mathbf{G}[id+j]$}, reading the corresponding aggregated gradients \textbf{$\mathbf{G^*}[j]$}, adding them together, and then writing them to aggregated gradient \textbf{$\mathbf{G^*}[j]$} again, for any $i\in [n]$ and $j\in [d]$.
For any two input data $I, I'$ of equal length, for any security parameter $\lambda$, $\mathbf{Accesses^{dense}}$ is identical and the statistical distance $\delta = 0$.
Finally, \textit{linear} algorithm is 0-statistical oblivious.
\end{proof}

Proof of Proposition \ref{prop:baseline}.
\begin{proof}
Let the access pattern observed through algorithm \ref{alg:baseline} be $\mathbf{Accesses^{baseline}}$, and it is as follows:
\begin{equation}
\nonumber
\small
\begin{split}
  \mathbf{Accesses^{baseline}} &= \\
  [(\mathbf{G}[1], \mathtt{read}&,*), (\mathbf{G_c^*}[1], \mathtt{write},*),...,
  (\mathbf{G_c^*}[d/c], \mathtt{write},*),..., \\
  (\mathbf{G}[k], \mathtt{read}&,*), (\mathbf{G_c^*}[1], \mathtt{write},*),...,
  (\mathbf{G_c^*}[d/c], \mathtt{write},*)]
 \end{split}
\end{equation}
\noindent
where $c$ is the number of gradients included in one cacheline and $\mathbf{G_c^*}$ is an array with $d/c$ cells where $\mathbf{G^*}$ is divided at the granularity of a cacheline.
Since $\mathbf{Accesses^{baseline}}$ is the identical sequence for any inputs of the same length, algorithm \ref{alg:baseline} is 0-statistical oblivious.
\end{proof}

\section{Relation with Differential Privacy}
\label{appendix:dp_olive}

\begin{table*}[]
\normalsize
% \small
\caption{Comparison with different schemes of DP-FL in terms of trust model and utility.}
\begin{center}
\renewcommand{\arraystretch}{1.0}
\begin{tabular}{lcc}
\hline
& Trust model & Utility \\ \hline \hline
CDP-FL \cite{geyer2017differentially, DBLP:conf/iclr/McMahanRT018, andrew2021nips, wei2020federated} & Trusted server & Good \\ \hline
LDP-FL \cite{zhao2020local, ldp2020, liu2020fedsel, ijcai2021-217} & Untrusted server & Limited \\ \hline
Shuffle DP-FL \cite{liu2020flame, girgis2021shuffled} & Untrusted server + Shuffler &  Shuffle DP-FL $\leq$ CDP-FL\\ \hline 
\textbf{\method (Ours)} & \textbf{Untrusted Server with TEE} &  \textbf{\method = CDP-FL} \\ \hline
\end{tabular}
\label{table:comparison}
\end{center}
\vspace{-10px}
\end{table*}

\subsection{Overview}

Differentially private FL (DP-FL) \cite{geyer2017differentially, DBLP:conf/iclr/McMahanRT018} has garnered significant attention due to its capacity to alleviate privacy concerns by ensuring Differential Privacy (DP) \cite{dwork2006differential}.
Researchers have explored various DP-FL techniques to strike a good balance between trust model and utility, as shown in Table \ref{table:comparison}.

In central DP Federated Learning (CDP-FL)  \cite{geyer2017differentially, DBLP:conf/iclr/McMahanRT018, andrew2021nips, wei2020federated}, a \textit{trusted} server collects the raw participants' data and takes the responsibility to privatize the global model.
(Client-level) CDP-FL guarantees that it is probabilistically indistinguishable whether a client is participating in the training or not.
It is defined as follows:
\begin{definition}[(client-level) $(\epsilon, \delta)$-differential privacy \cite{DBLP:conf/iclr/McMahanRT018}]
A randomized mechanism $\mathcal{M}:\mathcal{D}\rightarrow\mathcal{Z}$ satisfies $(\epsilon, \delta)$-DP if, for any two neighboring  datasets $D, D' \in \mathcal{D}$ such that $D'$ differs from $D$ in at most one client's record set and any subset of outputs $Z \subseteq \mathcal{Z}$, it holds that
\begin{equation}
\nonumber
  \Pr[\mathcal{M}(D)\in Z] \leq \exp(\epsilon) \Pr[\mathcal{M}(D')\in Z] + \delta.
\end{equation}
where $\mathcal{Z}$ corresponds to the final trained model and $\mathcal{M}(D)$ corresponds to the learning algorithm with perturbation (e.g., DP-SGD) that uses input client $D$'s training data to learn.
\end{definition}
\noindent
In general, CDP-FL provides a good trade-off between privacy and utility  (e.g., model accuracy) of differentially private models even at practical model scales \cite{DBLP:conf/iclr/McMahanRT018, andrew2021nips}.
However, CDP-FL requires the server to access raw gradients, which leads to major privacy concerns on the server as the original data can be reconstructed even from the raw gradients \cite{zhu2020deep, zhao2020idlg}.

In LDP (Local DP)-FL \cite{zhao2020local, ldp2020, liu2020fedsel, ijcai2021-217}, the clients perturb the gradients before sharing with an \textit{untrusted} server, guaranteeing formal privacy against both malicious third parties and the untrusted server.
LDP-FL does not require a trustful server unlike CDP-FL.
However, LDP-FL suffers from lousy privacy-utility trade-off, especially when the number of users is not sufficient (i.e., the signal is drowned in noise) or the number of the model parameters is large (i.e., more noises are needed for achieving the same level of DP).
Unfortunately, it is limited to models with an extremely small number of parameters or companies with a huge user base (e.g., 10 million).

To overcome the weakness of the utility of LDP by privacy amplification, a method using the shuffler model \cite{balle2019privacy, erlingsson2019amplification}, has been proposed \cite{liu2020flame}, i.e., Shuffle DP-FL.
This method introduces a trusted shuffler instead of trusting the server and achieves some level of utility.
However, clearly, it cannot outperform CDP in utility because we can simulate the shuffler mechanism on a trusted server.
The privacy amplification of the shuffler also has weaknesses, such as the need for a large number of participants and small parameter size due to the underlying LDP limitation.
This is clearly highlighted in Table 12 of \cite{erlingsson2020encode} \footnote{We can reproduce the similar result with our code \url{https://github.com/FumiyukiKato/FL-TEE/blob/master/src/eval-ldp-sgd.py}.}.
Hence, there is still a utility gap between CDP-FL and the state-of-the-art Shuffle DP-FL.

To fill this gap, our proposed \method can be used as illustrated in Figure \ref{fig:overview}.
\method employs TEE to ensure secure model aggregation on an untrusted server so that only differentially private models are observable by the untrusted server or any third parties.
The utility of \method is exactly the same as the conventional CDP-FL as the computation inside TEE can be implemented for arbitrary algorithms.
Note that there is differences from the pairwise-masking secure aggregation, which has limitations on the DP mechanism.
For example, it requires to discretize the parameters and noises and to add noises in a distributed manner \cite{kairouz2021distributed, chen2022fundamental}.

\subsection{DP-FL in \method: top-$k$ sparsified client-level CDP-FL on TEE}

\begin{algorithm}[t]
\small
\caption{DP-FL in \method}
\label{alg:olive_dp_fl}
\begin{algorithmic}[1]
\renewcommand{\algorithmicrequire}{\textbf{Input:}}
\renewcommand{\algorithmicensure}{\textbf{Output:}}

\Require $N$: \# participants, $q$: sampling rate of participants, $\eta_c$, $\eta_s$: learning rate, $\sigma$: noise parameter, $T$: number of rounds
\State $\text{KeyStore} \leftarrow \textit{Remote Attestation}$ with all user $i$ \Comment{key-value store in enclave that stores $sk_{i}$: user $i$'s shared key from RA in provisioning}
\Procedure{Train}{$q$, $\eta_c$, $\eta_s$, $\sigma$, $T$}
    \State Initialize model $\theta^0$, clipping bound $C$
    \For{each round $t=0, 1, \ldots, T$}
        \State $\mathcal{Q}^t \leftarrow$ (sample users with probability $q$) \Comment{securely in enclave}
        \For{each user $i \in \mathcal{Q}^t$ \textbf{in parallel}}
            \State $\text{Enc}(\Delta^{t}_i) \leftarrow$ \textsc{EncClient}$(i, \theta^t, \eta_c, C)$ \Comment{with AE mode}
            \State $\text{LoadToEnclave}(\text{Enc}(\Delta^{t}_i))$
            \State check if user $i$ is in $\mathcal{Q}^t$
            \State $sk_{i} \leftarrow \text{KeyStore}[i]$ \Comment{retrieve user $i$'s shared key}
            \State $\Delta^{t}_i \leftarrow \text{Decrypt}(Enc(\Delta^{t}_i), sk_{i})$ \Comment{with verification}
        \EndFor
        \State /* \textbf{Obliviously performed, such as Alg. \ref{alg:baseline} or \ref{alg:advanced}} */
        \Statex \;\;\;\;\;\;\;\;\;\, \red{$\mathbf{\tilde{\Delta}^t = \frac{1}{qN} \left( \sum_{i \in \mathcal{Q}^t} \Delta^{t}_i + \mathcal{N}(0, \sigma^2 C^2\mathbf{I}_d) \right)}$} \Comment{oblivious aggregation}\label{alg1:dp}
        \State $\text{LoadFromEnclave}(\tilde{\Delta}^t)$
        \State $\theta^{t+1} \leftarrow \theta^t + \eta_s \bar{\Delta}^t$
    \EndFor
\EndProcedure

\Procedure{EncClient}{$i$, $\theta^t$, $\eta$, $C$}
    \State $\theta \leftarrow \theta^t$
    \State $\mathcal{G} \leftarrow $ (user $i$'s local data split into batches)
    \For{batch $g \in \mathcal{G}$}
        \State $\theta \leftarrow \theta - \eta \nabla \ell(\theta;g)$
    \EndFor
    \State $\Delta \leftarrow \theta - \theta^t$
    \State \red{$\Delta$ $\leftarrow$ TopkSparse$(\Delta)$} \Comment{top-$k$ sparsification} \label{alg1:sparsification}
    \State \red{$\Delta' \leftarrow \Delta \cdot \min{\left(1, \frac{C}{||\Delta||_2}\right)}$} \Comment{$\ell2$ clipping}
    \State $\text{Enc}(\Delta') \leftarrow \text{Encrypt}(\Delta', sk_{i})$ \Comment{with shared key $sk_i$ from RA}
    \State \textbf{return} $\text{Enc}(\Delta')$
\EndProcedure

\end{algorithmic}
\end{algorithm}

Algorithm \ref{alg:olive_dp_fl} depicts the algorithm for the combination of CDP-FL and \method.
On the client side, after computing the parameter delta, top-$k$ sparsification is executed (line 21) followed by clipping (line 22), encryption, and data transmission to the TEE on the server side.
This approach just incorporates client-side top-$k$ sparsification into \texttt{DP-FedAVG} \cite{DBLP:conf/iclr/McMahanRT018}.
The hyperparameter $q$ is needed for privacy amplification through client-level sampling.
$\sigma$ is the noise multiplier that determines the variance of the Gaussian noise to satisfy DP (line 12) (which is noise's standard deviation divided by the clipping scale and commonly used in DP-SGD \cite{abadi2016deep} framework).
And $C$ is clipping parameter to bound $\ell2$-sensitivity.
A similar procedure has been proposed in  \cite{cheng2022differentially}, although TEE part is not included.

The privacy analysis of Algorithm \ref{alg:olive_dp_fl} is discussed in the rest of this section.
Recent works \cite{hu2022federated, cheng2022differentially} have investigated the combination of client-level CDP-FL and sparsification. 
The privacy analysis is performed by combining existing Renyi differential privacy (RDP) analysis techniques (or moments accountant \cite{abadi2016deep} which is equivalent to RDP analysis) as well as common CDP-FL \cite{DBLP:conf/iclr/McMahanRT018}.

However, one salient aspect is the treatment of sparsification (which is described in Section \ref{sec:background_federated_learning}).
The crucial point is whether the indices of the parameters selected by sparsification are common or distinct among all clients.
If all clients have common sparsified indices ($k$ out of $d$ indices), the Gaussian mechanism required for DP only needs $k$-dimensional noise, as only $k$ parameters of the global model require updating in a single round of aggregation.
This results in a direct reduction of noise by a factor of $O(k/d)$. 
To this end, \cite{hu2022federated} proposes a method for obtaining the common top-$k$ indices among many clients for sparsification.
However, as noted in \cite{ergun2021sparsified}, in practical setting, there is actually little overlap in the top-$k$ sparsified indices for each client, especially in the non-i.i.d. setting, which is general in FL. 
Hence, a common top-$k$ index appears to be impractical.

On the other hand, we consider the scenario where different sparsified indices are chosen for different clients.
This represents a standard setup in the absence of DP. 
In contrast to the previous case, where all clients shared a common set of sparsified indices, there is no reduction in Gaussian mechanism noise on the order of $O(k/d)$.
This is due to the fact that while each client transmits sparsified parameters of dimension $k$.
However, any of the $d$ dimensions of the global model may be updated with the transmitted sparsifed parameters.
Hence, noise need to be added to all dimensions to ensure DP rather than only to the $k$ dimensions.
This remains true regardless of whether the noise is added on the client or server side, or what type of sparsification is employed as far as aiming to guarantee a global model DP.
This may have been overlooked in previous work that employed sparsification \cite{liu2020flame}.

Nevertheless, despite the above discussion, such client-specific sparsification can improve the trade-off between privacy and utility to a certain extent.
This is because sparsification reduces the absolute value of the $\ell2$-norm of the transmitted parameters.
As we formally describe later, the $\ell2$-norm of the shared parameters from each client must be bounded by the clipping parameter $C$ to add Gaussian noise for DP.
When clipping is performed on the original dense parameters, all parameters contribute to the $\ell2$-norm.
In the case of sparsification, however, only $k$ parameters contribute to the $\ell2$-norm.
Intuitively, the less important $d - k$ parameters are discarded and the space in the $\ell2$-norm is allocated to the more important $k$ parameters, thus increasing their utility.
Consequently, this also means that the clipping size $C$ can be set lower in the sparsified case, which can lead to lower noise variance.
This observation is the basis for the sparsification proposed in \cite{cheng2022differentially}.
To be more precise, \cite{cheng2022differentially} sparsifies according to their own utility criteria, rather than selecting the top-$k$ parameters, but the characteristics of the privacy-utility trade-offs are the same.
In general, it can be concluded that the amount of noise required for CDP is the same in the case of sparsification as in the absence of sparsification.

\textbf{Formal privacy statement.}
We now formally state the DP satisfied by Algorithm \ref{alg:olive_dp_fl} 
for completeness.
The following definitions and lemmas are the same as the ones stated in existing studies such as \cite{hu2022federated, cheng2022differentially}.

\begin{definition}[Sensitivity]
The sensitivity of a function $f$ for any two neighboring inputs $D, D' \in \mathcal{D}$ is:
\begin{equation}
\nonumber
    \Delta_{f} = \sup_{D, D' \in \mathcal{D}} \|f(D)-f(D')\|.
\end{equation}
where $||\cdot||$ is a norm function defined in $f$'s output domain.
\end{definition}

\noindent
We consider $\ell2$-norm ($||\cdot||_{2}$) as $\ell2$-sensitivity for following analysis with Gaussian noise.
We use R\'{e}nyi DP (RDP) \cite{mironov2017renyi} because of the tightness of the privacy analysis and the composition.

\begin{definition}[$(\alpha, \rho)$-RDP \cite{mironov2017renyi}]
Given a real number $\alpha \in (1, \infty)$ and privacy parameter $\rho \ge 0$, a randomized mechanism $\mathcal{M}$ satisfies $(\alpha, \rho)$-RDP if for any two neighboring datasets $D, D' \in \mathcal{D}$ such that $D'$ differs from $D$ in at most one client's record set, we have that $D_{\alpha}(\mathcal{M}(D)||\mathcal{M}(D')) \le \rho$ where $D_{\alpha}(\mathcal{M}(D)||\mathcal{M}(D'))$ is the R\'{e}nyi
divergence between $\mathcal{M}(D)$ and $\mathcal{M}(D')$ and is given by
\begin{equation}
\nonumber
  D_{\alpha}(\mathcal{M}(D)||\mathcal{M}(D')) := \cfrac{1}{\alpha -1}\log{\mathbb{E}\left[\left(\cfrac{\mathcal{M}(D)}{\mathcal{M}(D')}\right)^{\alpha}\right]} \le \rho,
\end{equation}
where the expectation is taken over the output of $\mathcal{M}(D)$.
\end{definition}

\begin{lemma}[RDP composition \cite{mironov2017renyi}]
\label{lemma:rdp_composition}
If $\mathcal{M}_1$ satisfies $(\alpha, \rho_1)$-RDP and $\mathcal{M}_{2}$ satisfies $(\alpha, \rho_2)$, then their
composition $\mathcal{M}_{1} \circ \mathcal{M}_{2}$ satisfies $(\alpha, \rho_1 + \rho_2)$-RDP.
\end{lemma}

\begin{lemma}[RDP to DP conversion \cite{wang2019subsampled}]
\label{lemma:rdp_conversion}
If $\mathcal{M}$ satisfies $(\alpha, \rho)$-RDP, then it also satisfies $(\rho +
\frac{\log{(1/\delta)}}{\alpha - 1}, \delta)$-DP for any $0 < \delta < 1$.
\end{lemma}

\begin{lemma}[RDP Gaussian mechanism\cite{mironov2017renyi}]
\label{lemma:rdp_gaussian}
If $f: D \rightarrow \mathbb{R}^d$ has $\ell2$-sensitivity $\Delta_{f}$, then the Gaussian mechanism $G_{f}(\cdot) := f(\cdot) + \mathcal{N}(0, \sigma^2\Delta_{f}^2\mathbf{I}_{d})$ is $(\alpha, \alpha / 2\sigma^2)$-RDP for any $\alpha > 1$.
\end{lemma}

\begin{lemma}[RDP for subsampled Gaussian mechanism \cite{wang2019subsampled}]
\label{lemma:rdp_subsampled}
Let $\alpha \in \mathbb{N} $ with $\alpha \ge 2$ and $0 < q < 1$ be a subsampling ratio of subsampling operation $Samp_q$.
Let $G'_{f}(\cdot) := G_{f} \circ Samp_q (\cdot)$ be a subsampled Gaussian mechanism.
Then, $G'_{f}$ is $(\alpha, \rho'(\alpha, \sigma))$-RDP where
\begin{equation}
\begin{aligned}
\nonumber
\rho'(\alpha, \sigma) \le \cfrac{1}{\alpha - 1}\log\biggl(1 + 2q^2 \binom{\alpha}{2} &\min{\{2(e^{1/\sigma^2} - 1), e^{1/\sigma^2}\}} \\
&+ \sum_{j=3}^{\alpha}{2q^j \binom{\alpha}{j} e^{j(j - 1)/2\sigma^2}} \biggl).
\end{aligned}
\end{equation}
\end{lemma}

Finally, we state the formal differential privacy guarantees provided by Alg. \ref{alg:olive_dp_fl}.
\begin{theorem}[]
For any $\epsilon < 2 \log{(1/\delta)}$ and $0 < \delta < 1$, Alg. \ref{alg:olive_dp_fl} satisfies $(\epsilon, \delta)$-DP after $T$ communication rounds if
\begin{equation}
\nonumber
\sigma^2 \ge \cfrac{7q^2 T (\epsilon + 2 \log{(1/\delta)})}{\epsilon^2}.
\end{equation}
\end{theorem}

\begin{proof}
    In each round $t$ of $T$ in \textsc{Train} (line 2 of Alg. \ref{alg:olive_dp_fl}), let $f$ be a summation of delta parameters ($\Delta_{i}^{t}$, line 11), the $\ell2$-sensitivity of $f$ is $C$ due to clipping operation (line 22).
    As explained in detail above, this is independent of the sparsified dimension $k$.
    Hence, adding the Gaussian noise $\mathcal{N}(0, \sigma^2 C^2\mathbf{I}_{d})$, i.e., $G_{f}$, satisfies $(\alpha, \alpha / 2\sigma^2)$-RDP for any $\alpha > 1$ by Lemma \ref{lemma:rdp_gaussian}.
    Further, in the round, the participants are subsampled with probability $q$ (line 5).
    Then, following Lemma 3 of \cite{wang2019efficient}, if $\sigma^2 \ge 0.7$ and $\alpha \le 1 + (2/3) C^2 \sigma^2 \log{\frac{1}{q\alpha (1+\sigma^2)}}$, by Lemma \ref{lemma:rdp_subsampled}, subsampled Gaussian mechanism $G'_{f}(\cdot)$ satisfies $(\alpha, \frac{3.5 q^2 \alpha}{\sigma^2})$-RDP.
    Over $T$ rounds, by Lemma \ref{lemma:rdp_composition}, it satisfies $(\alpha, T\frac{3.5 q^2 \alpha}{\sigma^2})$-RDP.
    Lastly, we convert RDP guarantee to $(\epsilon, \delta)$-DP by Lemma \ref{lemma:rdp_conversion}.
    $\epsilon$ needs to hold $T\frac{3.5 q^2 \alpha}{\sigma^2} + \frac{\log{1/\delta}}{\alpha - 1} \le \epsilon$.
    Choose $\alpha = 1 + 2 \log{(1/\delta)}$, we obtain the final result.
\end{proof}

\subsection{Attack evaluation}

\begin{figure*}[t]
    \centering
    \includegraphics[width=0.95\hsize]{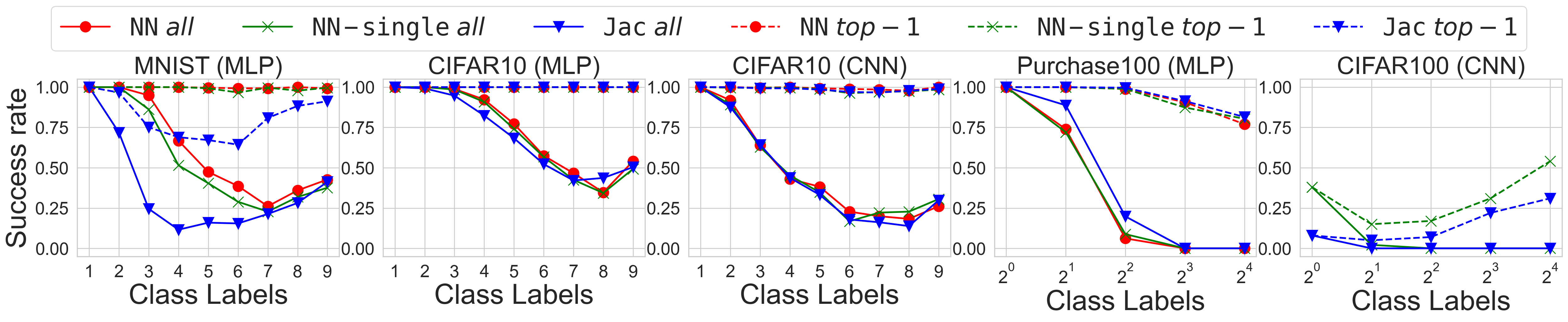}
    \vspace{-3px}
    \caption{Attack results on datasets with a fixed number of labels with DP ($\sigma=1.12$).}
    \label{fig:attack_fixed_label_with_dp}
\end{figure*}

\begin{figure*}[t]
    \centering
    \includegraphics[width=0.95\hsize]{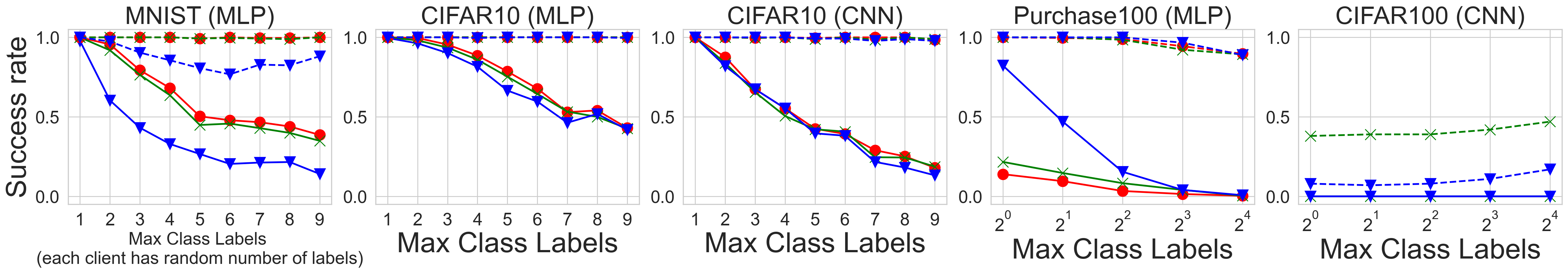}
    \vspace{-3px}
    \caption{Attack results on datasets with a random number of labels (more difficult setting) with DP ($\sigma=1.12$).}
    \label{fig:attack_random_label_with_dp}
\end{figure*}

\begin{figure}[t]
    \centering
    \includegraphics[width=0.9\hsize]{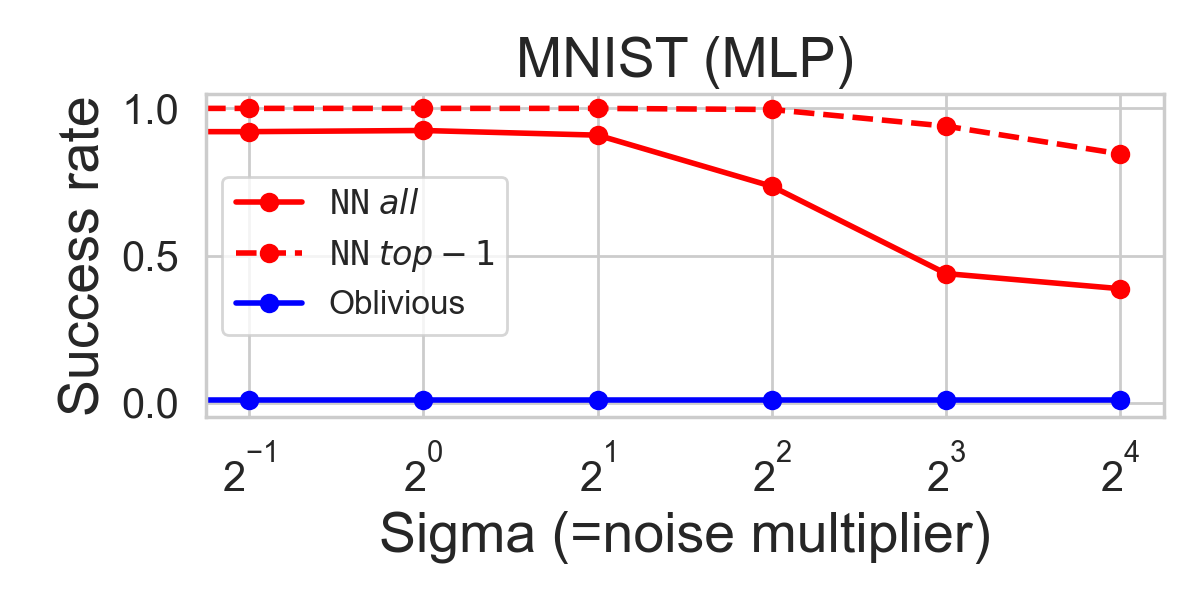}
    \caption{Attack performance with variable noise multiplier $\sigma$. At realistic noise scales, the attack performance remains high.}
    \label{fig:attack_per_noise}
\end{figure}

\begin{figure}[t]
    \centering
    \includegraphics[width=0.9\hsize]{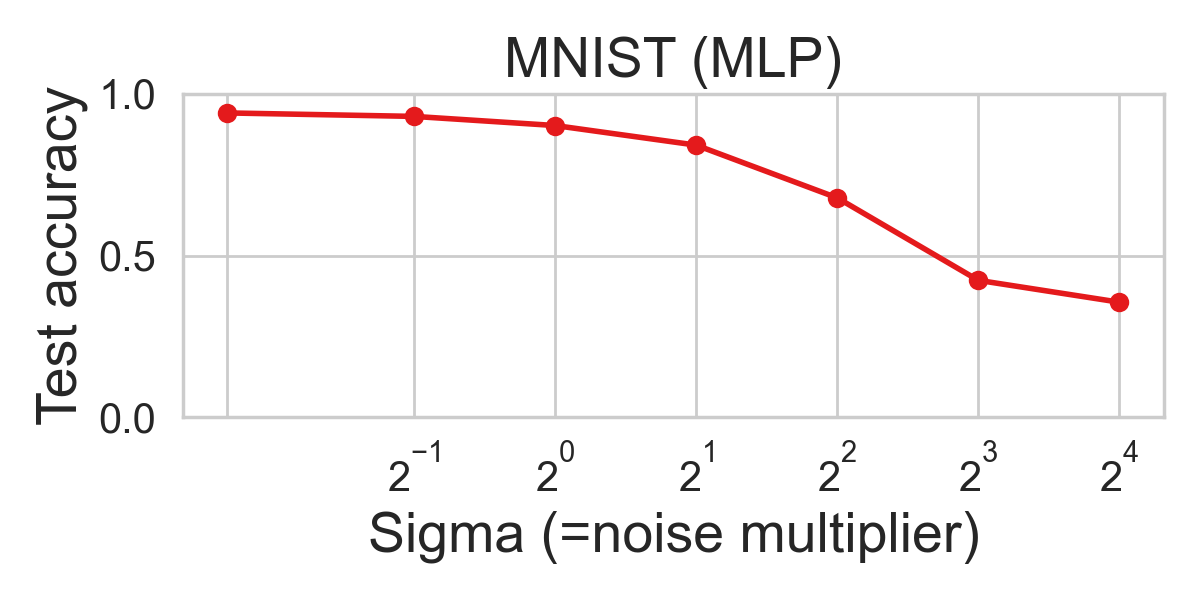}
    \caption{Effective noise scales in defensing do not provide sufficient utility.}
    \label{fig:utlity_per_noise}
\end{figure}

\begin{figure}[t]
  \begin{minipage}[b]{0.48\linewidth}
    \centering
    \includegraphics[width=0.9\hsize]{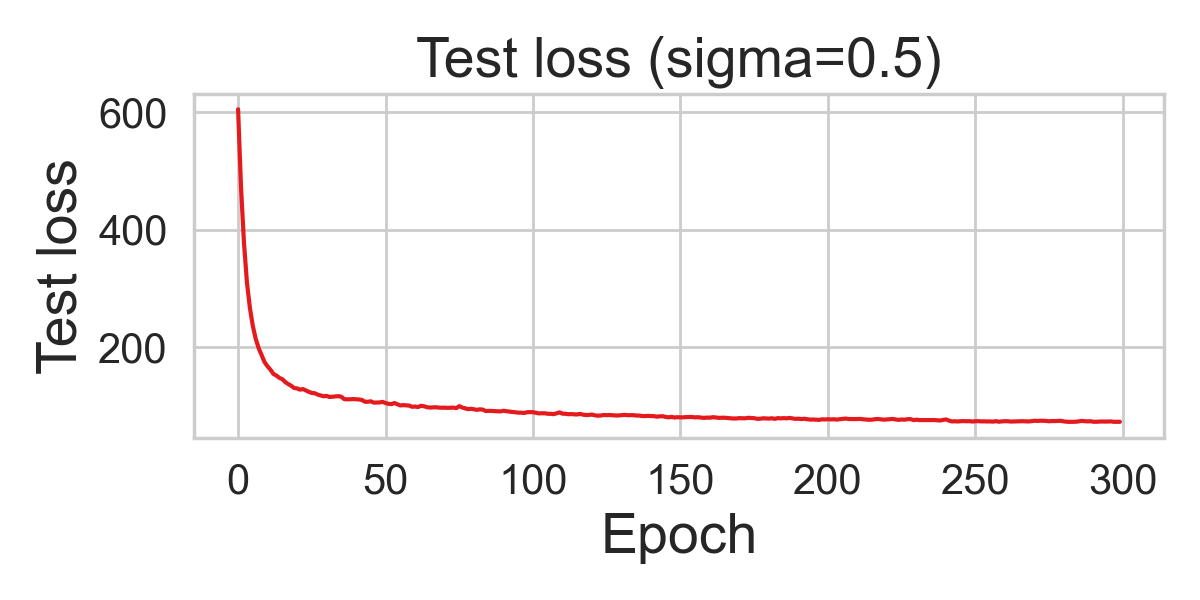}
  \end{minipage}
  \begin{minipage}[b]{0.48\linewidth}
    \centering
    \includegraphics[width=0.9\hsize]{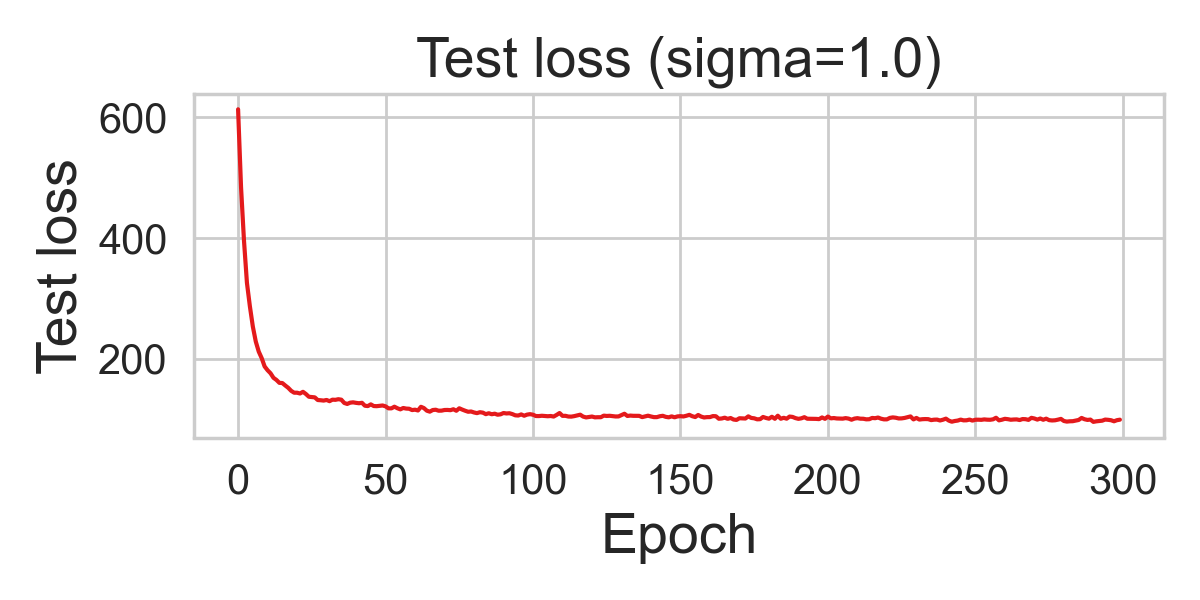}
  \end{minipage}
  \begin{minipage}[b]{0.48\linewidth}
    \centering
    \includegraphics[width=0.9\hsize]{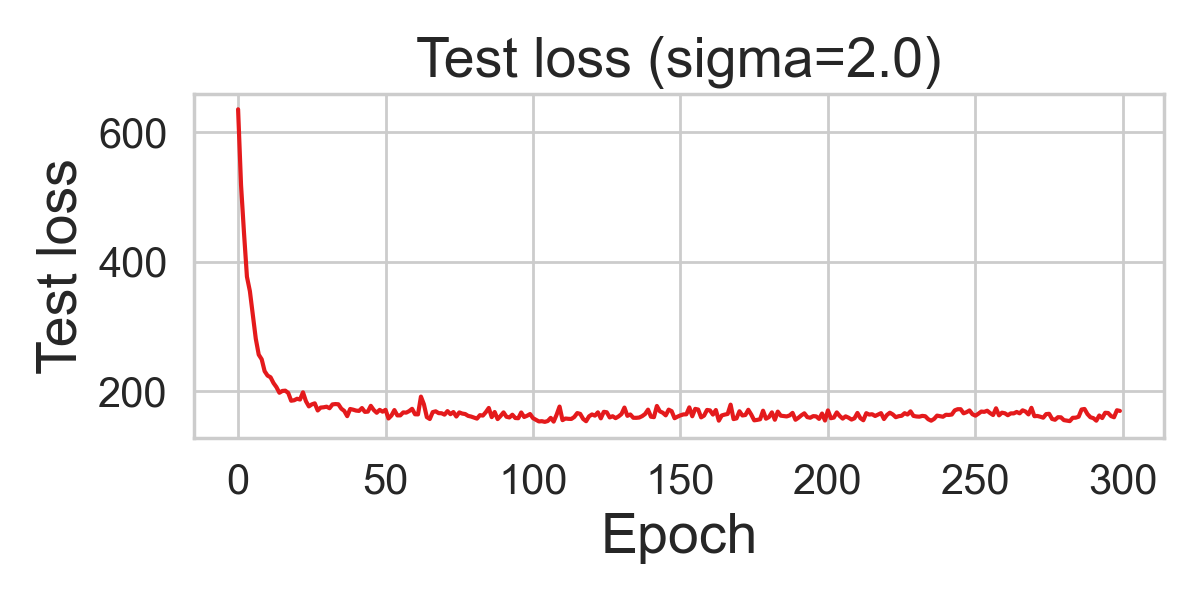}
  \end{minipage}
  \begin{minipage}[b]{0.48\linewidth}
    \centering
    \includegraphics[width=0.9\hsize]{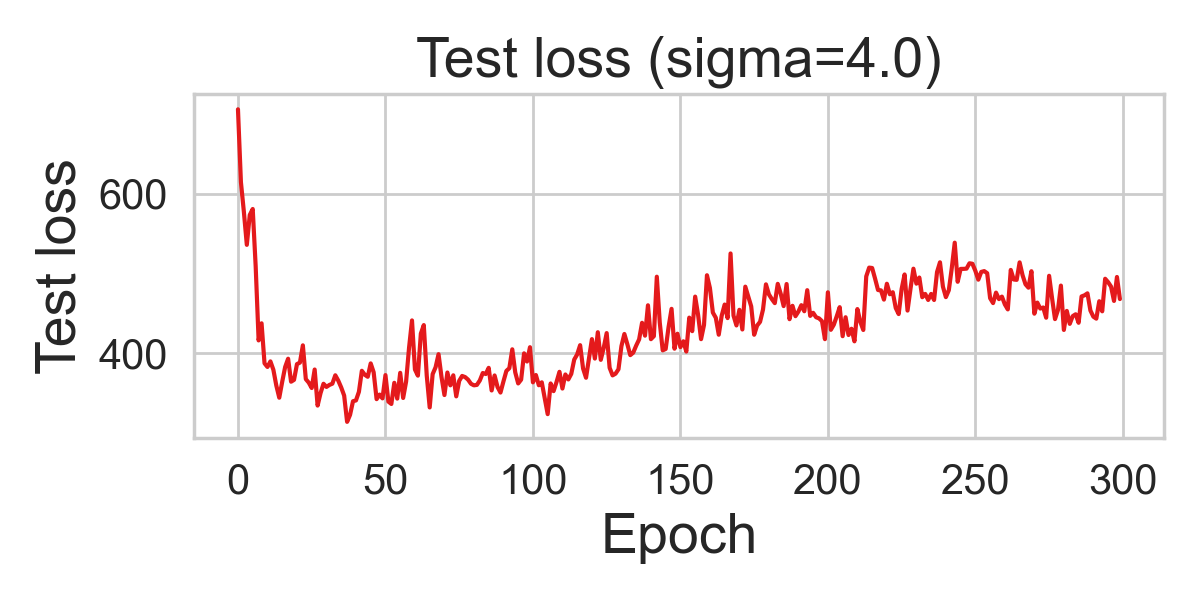}
  \end{minipage}
    \begin{minipage}[b]{0.48\linewidth}
    \centering
    \includegraphics[width=0.9\hsize]{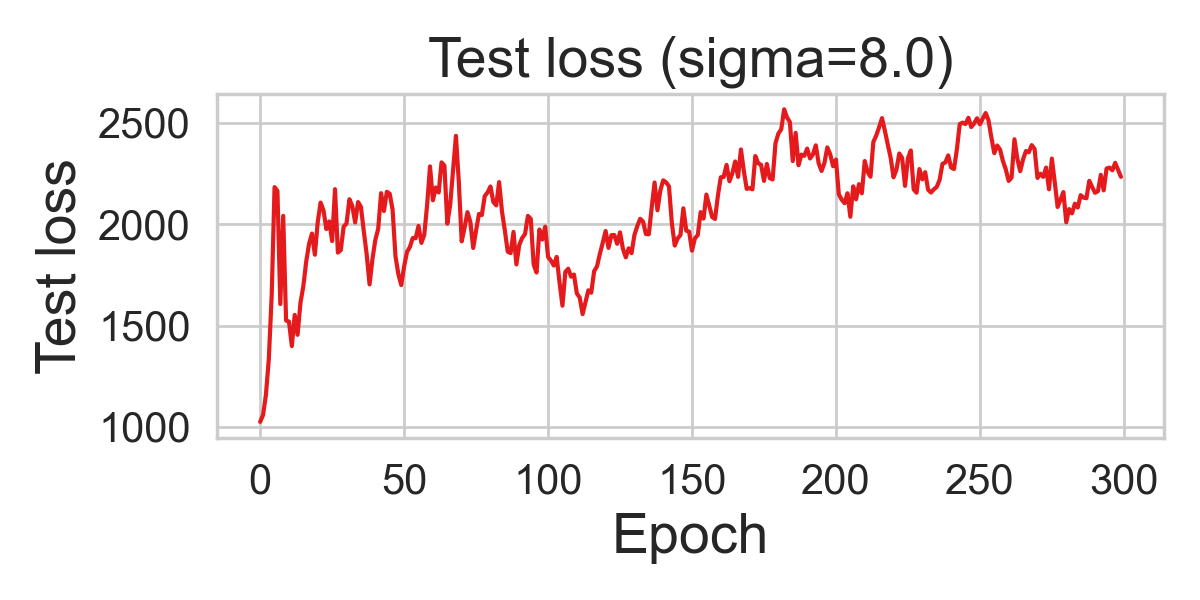}
  \end{minipage}
  \begin{minipage}[b]{0.48\linewidth}
    \centering
    \includegraphics[width=0.9\hsize]{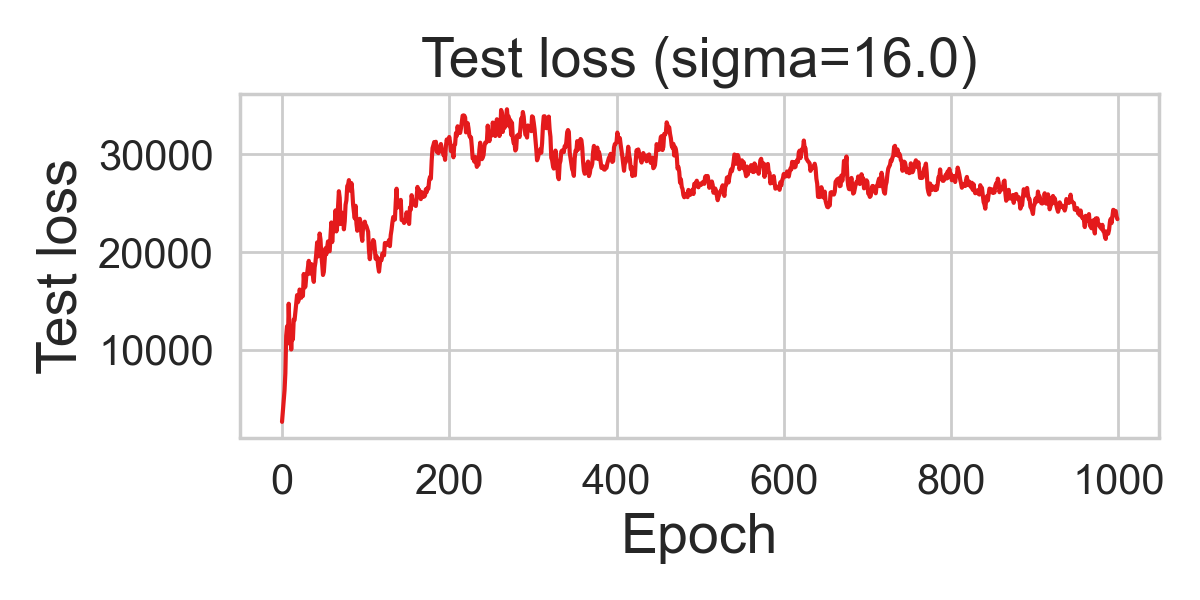}
  \end{minipage}
  \caption{Test losses for each noise multiplier $\sigma$.}
  \label{fig:test_losses_dp}
\end{figure}

\begin{table}[h]
\renewcommand{\arraystretch}{1.2}
% \normalsize
\caption{Architectures of the neural networks used as global models in all FL experiments in Sections \ref{sec:attack_experiment} and \ref{sec:experiment}. Readers can find the details of \textit{ResNet-18} at \url{https://github.com/weiaicunzai/pytorch-cifar100/blob/master/models/resnet.py}.}
\begin{center}
\begin{tabular}{lll}
\hline
Name            & Layers                                       & Details                                                                                                                                   \\ \hline
MNIST MLP       & \multicolumn{1}{c}{2 Fully Connected Layers} & \begin{tabular}[c]{@{}l@{}}Input: 28 * 28\\ Hidden: 64\\ Dropout: 0.5\\ Activation: ReLU\\ Output: 10\end{tabular}                     \\ \hline
CIFAR10 MLP     & 2 Fully Connected Layers                     & \begin{tabular}[c]{@{}l@{}}Input: 3 * 32 * 32\\ Hidden: 64\\ Dropout: 0.5\\ Activation: ReLU\\ Output: 10\end{tabular}                    \\ \hline
CIFAR10 CNN     & Convolutional 1                              & \begin{tabular}[c]{@{}l@{}}Input: 3 * 32 * 32\\ Activation: ReLU\\ Maxpooling:\\     kernel size: 2\\     stride: 2\end{tabular}          \\ \cline{2-3} 
                & Convolutional 2                              & \begin{tabular}[c]{@{}l@{}}Input: 6 * 14 * 14\\ Activation: ReLU\\ Maxpooling:\\     kernel size: 2\\     stride: 2\end{tabular}          \\ \cline{2-3} 
                & 3 Fully Connected Layers                     & \begin{tabular}[c]{@{}l@{}}Input: 16 * 5 * 5\\ Hidden1: 120\\ Activation: ReLU\\  Hidden2: 84\\  Activation: ReLU\\ Output: 10\end{tabular} \\ \hline
Purchase100 MLP & 2 Fully Connected Layers                     & \begin{tabular}[c]{@{}l@{}}Input: 600\\ Hidden: 64\\ Dropout: 0.5\\ Activation: ReLU\\ Output: 100\end{tabular}                           \\ \hline
CIFAR100 CNN    & \textit{ResNet-18} &                                                                                                                                             \\ \hline
\end{tabular}
\end{center}
\label{table:nnarc1}
\end{table}

\begin{table}[h]
\renewcommand{\arraystretch}{1.2}
\normalsize
\caption{Architectures of the neural networks used in Section \ref{sec:attack_on_index}. $d$ is the number of parameters of the global model trained in FL and $|L|$ is the number of labels of inference target.}
\begin{center}
\begin{tabular}{lll}
\hline
Name      & Layers                                       & Details                                                                                                         \\ \hline
\textsc{NN}        & \multicolumn{1}{c}{2 Fully Connected Layers} & \begin{tabular}[c]{@{}l@{}}Input: $d$\\ Hidden: 1000\\ Dropout: 0.5\\ Activation: ReLU\\  Output: $|L|$\end{tabular} \\ \hline
\textsc{NN-Single} & 2 Fully Connected Layers                     & \begin{tabular}[c]{@{}l@{}}Input: $d$\\ Hidden: 2000\\ Dropout: 0.5\\  Activation: ReLU\\ Output: $|L|$\end{tabular} \\ \hline
\end{tabular}
\end{center}
\label{table:nn_arc2}
\end{table}

Here, we demonstrate that our proposed attack remains viable even in the presence of differential privacy.
Firstly, we elucidate the reasons for our attack circumventing DP in Algorithm \ref{alg:olive_dp_fl}.
During each round of FL, the attacker is able to observe the index prior to perturbation (line 12 of Algorithm \ref{alg:olive_dp_fl}), thereby exposing the raw index information.
It should be noted that CDP-FL also employs distributed Gaussian noise on the client side.
However, it is performed after sparsification \cite{cheng2022differentially}, which implies that the raw index information is still visible.
Nevertheless, the randomization of the parameters of the global model by DP may reduce the accuracy of the attack.
This approach should be considered carefully, as the model may not be well trained itself.
In the next experiment, we see how much protection and how much model utility is sacrificed by the DP-based approach.

The experimental setting is the same as Section \ref{sec:attack_experiment}.
When the noise multiplier $\sigma$ is set to 1.12, the attack is essentially unaffected.
Figures \ref{fig:attack_fixed_label_with_dp} and \ref{fig:attack_random_label_with_dp} are DP versions of Figures \ref{fig:attack_fixed_label} and \ref{fig:attack_random_label}.
Although the success rate of attacks has decreased somewhat, there is almost no change.
Attacks are still possible.

In Figure \ref{fig:attack_per_noise}, we show the attack results on MLP of MNIST for increasing noise scale with fixed number of labels 3.
The horizontal axis indicates noise scale $\sigma$ by DP and the left-side start points indicate no noise.
Compared to the case with no noise, increasing the noise has less effect on the attack performance.
This makes sense from our attack design, where the attacker observes the raw index information of gradients even though the global model satisfies DP.
The blue line in the figure shows the attack success rate for oblivious algorithm (i.e., random inference by the attacker).
Since the number of labels is fixed at 3 and the total number of labels is 10, the success rate of this attack is $1/{}_{10} C_3 < 0.01$.
We can see that there is a limit to the defensive performance of the DP.
When we increase the noise multiplier ($\sigma$ is over $4.0$), defensive performance starts to increase, but such noise multiplier is over-strict in practical privacy degree.
This can be seen in Figure \ref{fig:utlity_per_noise}.
The figure shows the utility of the models trained with each noise multiplier, plotting the test accuracy when training MNIST with the MLP model.
The number of training rounds are fixed at 300, which is based on the observation that the training loss increased and did not converge with large multipliers (Figure \ref{fig:test_losses_dp}).
The results show that models trained with large noise multipliers are no longer useful, and that realistic noise does not protect against attacks.
These results highlight the importance of \method in CDP-FL.

\section{Running Example of Advanced}

\begin{figure}[t]
    \centering
    \includegraphics[width=1.02\hsize]{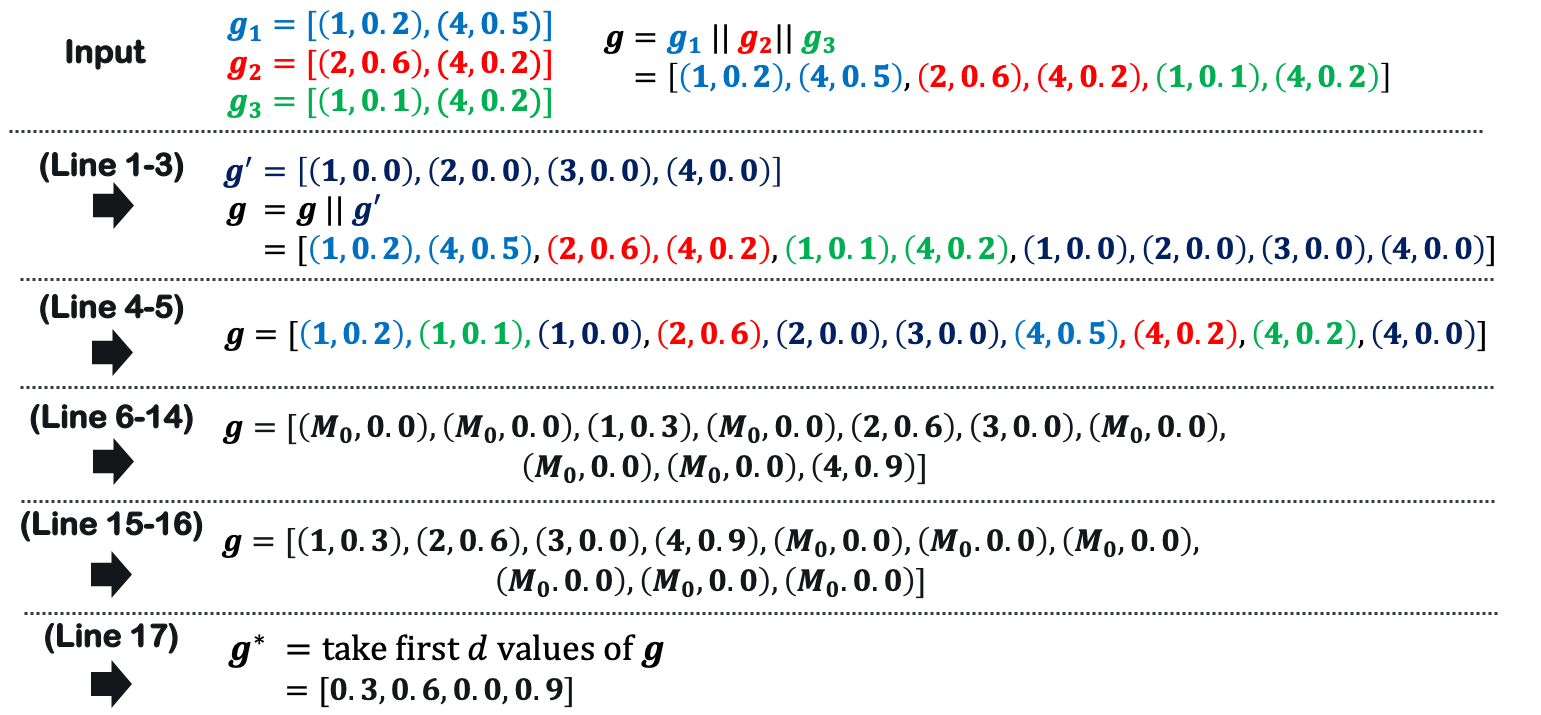}
    \caption{Running example of Advanced (Algorithm \ref{alg:advanced}) at $n=3$ (\#user), $k=2$ (\#sparsified dimension), $d=4$ (\#dimension).}
    \label{fig:running_example}
\end{figure}

We show a simple running example of Algorithm \ref{alg:advanced} at $n=3$, $k=2$ and $d=4$ in Figure \ref{fig:running_example}.

\section{Model architectures}
\label{appendix:model_arc}
Here are some details about the neural network model we used in our experiments.
The code for all models is available from our public repository.

Table \ref{table:nnarc1} shows the model used as the FL's global model throughout all experiments.
Table \ref{table:nn_arc2} describes the detailed design of the model used in the neural network-based attack in section \ref{sec:attack_experiment}.

\end{document}